\documentclass[11pt]{article}
\usepackage[letterpaper]{geometry}
\usepackage[parfill]{parskip}

\usepackage[utf8]{inputenc} 
\usepackage[T1]{fontenc}    
\usepackage[parfill]{parskip}
\usepackage{amsmath,amsthm,amssymb,bm,bbm}
\usepackage{mathtools}
\usepackage{cases}

\usepackage{microtype}

\usepackage{algorithm,algorithmic}
\usepackage{color}
\usepackage{appendix}
\usepackage{lmodern}  

\usepackage{url}
\usepackage[authoryear]{natbib}
\usepackage[colorlinks,citecolor=blue,urlcolor=blue,linkcolor=blue,linktocpage=true]{hyperref}

\usepackage{etoolbox}

\usepackage{cleveref}
\crefformat{equation}{(#2#1#3)}
\crefrangeformat{equation}{(#3#1#4) to~(#5#2#6)}
\crefname{equation}{}{}
\Crefname{equation}{}{}

\crefname{definition}{\textbf{definition}}{definitions}
\Crefname{definition}{Definition}{Definitions}
\crefname{assumption}{\textbf{assumption}}{assumptions}
\Crefname{assumption}{Assumption}{Assumptions}

\definecolor{maroon}{RGB}{192,80,77}

\newtheorem{theorem}{Theorem}
\newtheorem{lemma}[theorem]{Lemma}

\newtheorem{corollary}[theorem]{Corollary}
\newtheorem{definition}[theorem]{Definition}

\newtheorem{assumption}{Assumption}

\newcommand{\argmax}{\mathop{\mathrm{argmax}}}

\def\E{\mathbb{E}}
\def\P{\mathbb{P}}

\def\R{\mathbb{R}}

\def\cE{\mathcal{E}}
\def\cF{\mathcal{F}}

\def\cN{\mathcal{N}}

\usepackage{etoolbox}
\usepackage{comment}
\usepackage{caption}
\usepackage{subcaption}
\usepackage{isomath}
\usepackage{indentfirst}
\usepackage{enumerate}
\usepackage{enumitem}
\setlist{leftmargin=10mm}
\usepackage{mathrsfs}
\usepackage{multirow}
\usepackage{amstext}
\usepackage{dsfont}
\def\ind{\mathds{1}}
\def\vr{\mathbf{r}}
\newbool{twocol}
\setbool{twocol}{false}

\newbool{compact}  
\setbool{compact}{true}

\usepackage{booktabs}       
\usepackage{amsfonts}       
\usepackage{nicefrac}       
\usepackage{xcolor}         

\usepackage{authblk}

\usepackage{array}
\newcolumntype{L}{>{\centering\arraybackslash}m{3cm}}
\usepackage{makecell}

\title{Towards Agnostic Feature-based Dynamic Pricing: Linear Policies vs Linear Valuation with Unknown Noise}

\author{Jianyu Xu }
\author{Yu-Xiang Wang }
\affil{Department of Computer Science\\
	University of California, Santa Barbara\\
\texttt{\{xu\_jy15, yuxiangw\}@ucsb.edu} }

\begin{document}
	
\maketitle

\begin{abstract}

In feature-based dynamic pricing, a seller sets appropriate prices for a sequence of products (described by feature vectors) on the fly by learning from the binary outcomes of previous sales sessions (``Sold'' if valuation $\geq$ price, and ``Not Sold'' otherwise). Existing works either assume \emph{noiseless} linear valuation or \emph{precisely-known} noise distribution, which limits the applicability of those algorithms in practice when these assumptions are hard to verify. In this work, we study two more agnostic models: (a) a ``linear policy'' problem where we aim at competing with the best linear pricing policy while making no assumptions on the data, and (b) a ``linear noisy valuation'' problem where the random valuation is linear plus an unknown and assumption-free noise. For the former model, we show a $\tilde{\Theta}(d^{\frac13}T^{\frac23})$ minimax regret up to logarithmic factors. For the latter model, we present an algorithm that achieves an $\tilde{O}(T^{\frac34})$ regret, and improve the best-known lower bound from $\Omega(T^{\frac35})$ to $\tilde{\Omega}(T^{\frac23})$. These results demonstrate that no-regret learning is possible for feature-based dynamic pricing under weak assumptions, but also reveal a disappointing fact that the seemingly richer \emph{pricing feedback} is not significantly more useful than the \emph{bandit-feedback} in regret reduction.

	
\end{abstract}
\section{INTRODUCTION}
\label{sec_intro}
	In a dynamic pricing process, a seller presents prices for the products and adjusts these prices according to customers' feedback (i.e., whether they decide to buy or not) to maximize the revenue. Existing works on the single-product pricing problem \citep{kleinberg2003value, wang2021multimodal} assume that customers make decisions only according to the comparisons between prices and their own (random) valuations, and the goal is to find out a best fixed price that maximizes the (expected) revenue. In general, the single-product pricing problem has been well studied under a variety of assumptions.

However, these methods are not applicable when there are thousands of highly differentiated products with no experience in selling them. This motivates the idea of ``contextual pricing'' \citep{cohen2020feature_journal, mao2018contextual, javanmard2019dynamic, liu2021optimal}, where each sale session is described by a context that also affects the valuation and pricing.

\fbox{\parbox{0.98\textwidth}{Contextual pricing. For $t=1,2,...,T:$
		\small
		\noindent
		\begin{enumerate}[leftmargin=*,align=left]
			\setlength{\itemsep}{0pt}
			\item A context $x_t\in\R^{d}$ is revealed that describes a sales session (product, customer and context).			
			\item The customer valuates the product as $y_t$ using $x_t$.
			\item The seller proposes a price $v_t>0$ concurrently (according to $x_t$ and historical sales records).
			\item The transaction is successful if $v_t\leq y_t$, i.e., the seller gets a reward $r_t = v_t\cdot\ind(v_t\leq y_t)$. 
		\end{enumerate}
	}
}

Here $T$ is the time horizon known to the seller in advance\footnote{Here we assume $T$ known for simplicity of notations. In fact, if $T$ is unknown, then we may apply a ``doubling epoch'' trick as \citet{javanmard2019dynamic} and the regret bounds are the same.}, $x_t$'s can be either stochastic (i.e., each $x_t$ is independently and identically distributed) or adversarial (i.e., the sequence $\{x_t\}_{t=1}^T$ are arbitrarily chosen and fixed by nature before $t=0$), and $\ind_t:=\ind(v_t\leq y_t)$ is an indicator that equals $1$ if $v_t\leq y_t$ and $0$ otherwise.  
In this work, we consider two distinct problem setups that make use of the feature vector $x_t$.

\begin{enumerate}[label=(\alph*)]
	\item Linear Policy (\textbf{LP}): $(x_t,y_t)$'s are selected by nature (or an oblivious adversary) arbitrarily, and the learning goal is to compete with 
	the optimal linear prices $v_t^*=x_t^{\top}\beta^*$ where $\beta^*$ maximizes the cumulative reward in the hindsight. 
	\item Linear Valuation (\textbf{LV}): assume \emph{valuations} are linear $+$ noise, i.e., $y_t = x_t^{\top}\theta^*+N_t$, where $\theta^*\in\R^d$ is a fixed vector and $N_t$ is a market noise,
	drawn i.i.d. from a fixed \emph{unknown} distribution $\mathbb{D}$. The learning goal is to compete with the \emph{globally} optimal price $v_t^*=\argmax_v v\cdot \Pr[v\leq y_t|x_t]$ with no restrictions on the pricing policy.
\end{enumerate}

These two problem setups --- although quite similar at a glance --- are intrinsically different. The LP problem makes no assumptions on the $x_t\rightarrow y_t$ mapping, i.e., agnostic learning. Customers' valuations are not necessarily linear (and can be deterministic/noisy/stochastic/adversarial), but the seller competes with the optimal policy in a constrained family. In contrast, the LV problem makes mild modeling assumptions about the distribution of $y_t$ given $x_t$ while keeping the policy class unrestricted. In other words, LP is modeling our strategy while LV is modeling the nature. We adopt \emph{regret} as a metric of algorithmic performance: For the LP problem, we compare its (expected) reward with that of the optimal fixed $\beta^*$ in hindsight (i.e., an \emph{ex post} regret); For the LV problem, we compare its (expected) reward with the largest expected reward condition on $\theta^*$ and $\mathbb{D}$ (i.e., an \emph{ex ante} regret). We will clarify the difference between LP and LV in Appendix \ref{appendix_lp_vs_lv} with more details and examples. We emphasize that in both settings, the distributions of the valuation are unknown and non-parametric, and we are interested in designing  no-regret algorithms and characterizing the complexity.


\begin{table*}[]
	\centering
	\caption{Summary of existing regret bounds and our results}
	\label{table_related_works}
	\resizebox{\textwidth}{!}{
		\begin{tabular}{|l|L|L|L|L|L|}
			\hline
			Problem          & \multicolumn{4}{c|}{\textbf{Linear Valuation (LV)}}                                                                                                                                                                                                                                                                                  & \multirow{2}{3cm}{\centering \textbf{Linear Policy (LP)}}          \\ \cline{1-5}
			Noise Assumption & No Noise                                                                            & Known, Log-concave                                                                & Parametric                                                              & \textbf{Agnostic, Bounded}                                       &                                                                    \\ \hline
			Upper Bound      & $O(d\log\log{T})$ \quad \citep{leme2018contextual}   & $O(d\log{T})$ \quad\quad\quad  \citep{xu2021logarithmic}        & $\tilde{O}(d\sqrt{T})$\quad \quad\quad  \citep{wang2021dynamic}& $\tilde{O}(T^{\frac34}+d^{\frac12}T^{\frac58})$   \textbf{ This Work}                & $\tilde{O}(d^{\frac13}T^{\frac23}) $   \quad \quad\quad  \textbf{ This Work}      \\ \hline
			Lower Bound      & $\Omega(d\log\log{T})$ \citep{kleinberg2003value}& $\Omega(d\log{T})$ \citep{javanmard2019dynamic}& $\Omega(d\sqrt{T})$ \quad \quad\quad \quad  \citep{ban2021personalized}     & {$\tilde{\Omega}(T^{\frac23})$  \quad \quad \quad \quad \citep{kleinberg2003value} and \textbf{ This Work}          } & {$\tilde{\Omega}(d^{\frac13}T^{\frac23})$ \quad \quad\quad   \textbf{ This Work}   }         \\ \hline
		\end{tabular}
	}
\end{table*}

\ifbool{compact}{\noindent\textbf{Summary of Results.}}{\paragraph{Summary of Results.}} Our contributions are threefold.
\begin{enumerate}
	\item For the LP problem with adversarial $x_t$'s, we present an algorithm ``Linear-EXP4'' that achieves $\tilde{O}(d^{\frac13}T^{\frac23})$ regret.
	\item For the LV problem with adversarial $x_t$'s, we present an algorithm ``D2-EXP4'' that achieves $\tilde{O}(T^{\frac34}+d^{\frac12}T^{\frac58})$. 
	\item We present an $\tilde{\Omega}(d^{\frac13}T^{\frac23})$ regret lower bound for LP problem and an $\tilde{\Omega}(T^{\frac23})$ for LV problem (even with stochastic $x_t$'s, known $\theta^*$ and Lipschitz valuation distribution).  The results indicate ``Linear-EXP4'' optimal up to logarithmic factors. 
\end{enumerate}
To the best of our knowledge,  we are the first to study the LP problem and the version of the LV problem with no assumption on the noise. Comparing to the existing literature on this problem \citep{cohen2020feature_journal,javanmard2019dynamic}, our model makes fewer assumptions. Our results for LP is information-theoretically optimal, and our results in LV improve over the best known upper and lower bounds (from $\tilde{O}(T^{\frac23\vee(1-\alpha)})$  on i.i.d. $x_t$'s with an indeterministic $\alpha$ and $\Omega(T^{\frac35})$ in \cite{luo2021distribution}).  

\ifbool{compact}{\noindent\textbf{Technical Novelty.}}{\paragraph{Technical Novelty.}} In this work, we make use of the \emph{half-Lipschitz} nature in pricing problems: the probability of a price to be accepted will not decrease as the price decreases. This has been used in \citet{kleinberg2003value} and \citet{cohen2020feature_journal}. However, they directly applied this property in discretizing the action and policy spaces, which would lead to a linear regret in our LV problem setting. In our algorithm D2-EXP4, we settle this issue by also discretizing the noise distribution space and include these discretized CDF's as part of policy candidates. We also carefully adopt a conservative ``markdown''\footnote{A price markdown is defined as a reduction on the selling price.} on the discretized output price to ensure a large-enough probability of acceptance. In this way, we get rid of all assumption on the noise distribution (even the basic Lipschitzness assumed by \citet{luo2021distribution}) while achieving a sub-linear regret. This discretization method, along with the price markdown, can be easily transferred to any pricing problem settings with unknown i.i.d. noise. For the lower bound proof, we adapt the nested intervals and bump functions introduced by \citet{kleinberg2004nearly} for continuum bandits to our pricing problem models, and extend the $\Omega(T^{\frac23})$ regret lower bound on non-continuous demand functions \citep{kleinberg2003value} to Lipschitz ones.

\section{RELATED WORKS}
\label{sec_related_works}
In this section, we discuss how our work relates to the existing literature on (contextual/non-contextual) pricing, bandits, and contextual search. 
	
\paragraph{Non-Contextual Dynamic Pricing.}
Dynamic pricing was extensively studied under the single-product (non-contextual) setting \citep{kleinberg2003value, besbes2009dynamic, besbes2012blind, wang2014close, besbes2015surprising, chen2019nonparametric, wang2021multimodal}. The crux of pricing is to learn the demand curve (i.e., the noise distribution  in our LP problem) from Boolean-censored feedback. \citet{wang2021multimodal} concludes existing results and characterizes the impact of different assumptions on the demand curve on the minimax regret. The problem of contextual dynamic pricing is more challenging mainly because we need to learn the valuation parameter $\theta^*$ and the noise distribution jointly. Knowing one would imply a learning algorithm for another \citep{javanmard2019dynamic, luo2021distribution}, but learning both together makes the problem highly nontrivial.

	
\paragraph{Contextual Dynamic Pricing.}
	There is a growing body of recent works focusing on the LV model of the contextual dynamic pricing problem \citep{cohen2020feature_journal, javanmard2019dynamic, xu2021logarithmic, luo2021distribution, fan2021policy}, but most of them make strong assumptions about the noise. Table \ref{table_related_works} lists the best existing results under these assumptions. Besides these works, \citet{cohen2020feature_journal} also achieved an $O(d \log T)$ regret when the variance of the Sub-Gaussian noise is extremely small, i.e., $\tilde{O}(1/T)$. It is worth mentioning that our ``Linear-EXP4'' shares the same discretization factor with ``ShallowPricing'' algorithm in \citet{cohen2020feature_journal}, but ours solves a different problem. The closest works to ours are the recent \citet{luo2021distribution} and \citet{fan2021policy} that study the LV problem under only smoothness and log-concavity assumptions. In \citet{luo2021distribution}, they develop a UCB-style algorithm that achieves $\tilde{O}(T^{\frac23\vee(1-\alpha)})$ regret for noises with $2^{\text{nd}}$-order smooth and log-concave CDF's, assuming the existence of a good-enough estimator that might approach $\theta^*$ with $O(T^{-\alpha})$ error only with the logged data. However, such an estimator was neither described nor trivial to construct with $\alpha >0$. In \citet{fan2021policy}, they present a two-phase algorithm, with an exploration phase followed by an exploitation phase, and achieves $\tilde{O}((Td)^{\frac{2m+1}{4m-1}})$ regret for noises with $m^{\text{th}}$-order smooth ($m\geq2$) and ``well-behaved''\footnote{A property defined similarly as log-concavity.} CDF's. In comparison, our ``D2-EXP4'' algorithm achieves an $\tilde{O}(T^{\frac34})$ regret with no distributional assumptions such as Lipschitzness or smoothness.

\paragraph{Bandits}
	A multi-armed bandit (MAB) is an online learning model where one can only observe the feedback of the selected action at each time. Both LP and LV can be reduced to contextual bandits \citep{langford2007epoch, agarwal2014taming} as long as the policies and prices are finite. In this work, we make use of an ``EXP-4'' algorithm \citep{auer2002nonstochastic} in a new way: By carefully discretizing the parameter space and distribution functions, we enable EXP-4 agents to find out near-optimal policies among infinite continuum policy spaces. There exists another family of bandit problem: continuum-armed bandit (CAB) \citep{agrawal1995continuum, kleinberg2004nearly, auer2007improved}, where the action space is continuum and the reward function is Lipschitz. In this work, we adapt the (bump functions, nested intervals) structures in \citet{kleinberg2004nearly} to our lower bound proof. This adaptation is non-trivial since (1) their reward functions is not suitable for pricing problems, and (2) their feedback is not Boolean-censored. 
	
	
	Our results on the LP problem reveal that a reduction to contextual bandits is ``tight'' in regret bounds. A similar situation also occurs in \citet{kleinberg2003value} on non-contextual pricing. These results indicate a pricing feedback is not substantially richer than a bandit feedback in information theory, which is surprising as a pricing feedback indicates the potential feedback of a ``halfspace'' rather than a single point. However, does this imply we cannot get any extra information from a pricing feedback? Notice that we are matching a no-Lipschitz upper bound with a Lipschitz lower bound! In fact, a revenue curve is naturally ``half Lipschitz'', which helps us get rid of this assumption. We will discuss this property in Paragraph \ref{paragraph_conservative}. 
	
	
	\paragraph{Contextual search}
	Contextual pricing is cohesively related to contextual search problems \citep{leme2018contextual, lobel2018multidimensional, liu2021optimal, krishnamurthy2020contextual} where they also learn from Boolean feedback and usually assume linear contexts. However, they are facing slightly different settings: \citet{leme2018contextual, lobel2018multidimensional} are noiseless and could achieve an optimal $O(\log\log{T})$ regret; \citet{liu2021optimal} allows noises directly on customers' decisions instead of the valuations in our setting; \citet{krishnamurthy2020contextual} allows only small-variance valuation noises that is similar to \citet{cohen2020feature_journal}.

\section{PRELIMINARIES}
\label{sec_preliminaries}
\paragraph{Symbols and Notations.} Now we introduce the mathematical symbols and notations involved in the following pages. The game consists of $T$ rounds. $x_t, \beta^*, \theta^*\in\R^d_+, y_t, N_t\in\R, v_t\in\R_+$\footnote{We do not assume $y_t\geq0$ since some customer would not buy anything despite the price.}, where $d\in\mathbb{Z}_+$. At each round, we receive a payoff (\emph{reward}) $r_t=v_t\cdot\ind_t$ where $\ind_t:=\ind(v_t\leq y_t)$ indicates the acceptance of $v_t$, i.e., $\ind_t=1$ if $v_t\leq y_t$ and 0 otherwise. For LP problem, we denote $F_{LP}(v|x)$ as a \emph{demand function}, i.e. the probability of price $v$ being accepted given feature $x$. Therefore, $F_{LP}(v|x)$ is non-increasing with respect to $v$, for any $x\in\R^d$. For LV problem, we specifically denote $u_t=x_t^{\top}\theta^*$ as the \emph{noiseless valuation} (or \emph{expected valuation} for zero-mean noises), and denote $F$ as its CDF. Finally, we define $h(v,x) = v\cdot F_{LP}(v|x)$ as an \emph{expected revenue} function of price $v$ given feature $x$ in an LP problem, and $g(v, u, F):=v\cdot(1-F(v-u))$ as an \emph{expected revenue} function of price $v$ given any noiseless valuation $u$ and noise distribution $F$ in an LV problem.

We may use discretization methods in the following sections. Here we adopt the notation in \citet{cohen2020feature_journal} by denoting
\begin{equation}
\lfloor{x}\rfloor_{\gamma}:=\lfloor\frac{x}{\gamma}\rfloor\cdot\gamma, \lceil{x}\rceil_{\gamma}:=\lceil\frac{x}{\gamma}\rceil\cdot\gamma.
\label{equ_discretization}
\end{equation}
as the \emph{$\gamma$-lower/upper rounding} of $x$, which discretize $x$ as its nearest smaller/larger integer multiples of $\gamma$. Similarly, for $\theta\in\R^d$, we may define $\lfloor{\theta}\rfloor_{\gamma}:=[\lfloor{\theta_1}\rfloor_{\gamma}, \lfloor{\theta_2}\rfloor_{\gamma}, \ldots, \lfloor{\theta_d}\rfloor_{\gamma}]^{\top}$ and $\lceil{\theta}\rceil_{\gamma}:=[\lceil{\theta_1}\rceil_{\gamma}, \lceil{\theta_2}\rceil_{\gamma}, \ldots, \lceil{\theta_d}\rceil_{\gamma}]^{\top}$. Based on this, we define a counting set $N_{\gamma, a}:=\left\{0,1,2,\ldots, \lfloor\frac{a}{\gamma}\rfloor\right\}$. 

\paragraph{Regret Definitions.} Next we define the regrets in both problems.
\begin{definition}[Regret in LP]
	We define $Reg_{LP}$ as the regret of the Linear Policy pricing problem.
	\begin{equation}
	Reg_{LP}:=\max_{\beta}\sum_{t=1}^{T} h(x_t^{\top}\beta, x_t) - h(v_t, x_t).
	\label{equ_def_reg_lp}
	\end{equation}
	\label{def_reg_lp}
\end{definition}
\begin{definition}[Regret in LV]
	We define $Reg_{LV}$ as the regret of the Linear Noisy Valuation problem.
	\begin{equation}
	Reg_{LV}:=\sum_{t=1}^{T}\max_v g(v, u_t, F) - g(v_t, u_t, F).
	\label{equ_def_reg_lv}
	\end{equation}
	\label{def_reg_lv}
\end{definition}
Again, we aim at competing with the best fixed $\beta^* = \argmax_{\beta}\sum_{t=1}^{T} h(x_t^{\top}\beta, x_t)$ in an LP problem, and with the global best pricing policy (maximizing expected revenue at every $t$) in an LV problem.

\paragraph{Summary of Assumptions} We specify the problems by the following assumptions:
\begin{assumption}[bounded features and parameters]
	Without losing generality, we assume that $x_t, \beta^*, \theta^*\in\R^d_+,\|x_t\|_2\leq B, \|\beta^*\|_2\leq1, \|\theta^*\|_2\leq1$, where $B\in\mathbb{Z}^+$ is a constant known to us in advance.
	\label{assump_bounded_features_parameters}
\end{assumption}
\begin{assumption}[decreasing demand in LP]
	In LP problem, assume that $F_{LP}(v|x)$ is non-increasing for any $v\geq0, x\in\R^{d}_+$.
	\label{assump_decrease_demand_LP}
\end{assumption}
\begin{assumption}[bounded noise]
	In LV problem, assume that $N_t\in[-1,1]$ that is i.i.d. sampled from a fixed unknown distribution $\mathbb{D}$.
	\label{assump_bounded_noise}
\end{assumption}
These assumptions are mild and common for algorithm design.  Based on these assumptions above, we only have to consider prices in $[0, B]$ for LP problems and $[0, B+1]$ for LV problems. Besides, we assume that $T\geq d^4$ for a simplicity of comparing among different terms in regret bounds. In Section \ref{subsec_lower_bound}, we will introduce more assumptions to the distribution functions to demonstrate that our lower bounds hold \emph{even if} those assumptions are made.



\section{ALGORITHM}
\label{sec_algorithm}
In this section, we propose two algorithms, Linear-EXP4 and D2-EXP4, for LP and LV problems respectively. Both of them are based on the EXP-4 algorithm \citep{auer2002nonstochastic} along with discretized policy sets. First of all, we define these policy sets:

\begin{definition}[parameter set]
	For any small $0<\Delta<1$, we define a parameter set $\Omega_{\Delta, d}\subset\R^d$:
	\begin{equation*}
	\begin{aligned}
	\Omega_{\Delta, d}:=&\left\{\|\theta\|_2\leq1, \theta = [n_1\Delta, n_2\Delta, \ldots, n_d\Delta]^{\top},n_1, n_2, \ldots, n_d\in N_{\Delta, 1}\right\}
	\end{aligned}
	\label{equ_parameter_set}
	\end{equation*}
	\label{def_parameter_set}
\end{definition}

\begin{definition}[CDF set]
	For any small $0<\gamma<1$, we define a Cumulative Distribution Function (CDF) set $\cF_{\gamma}$:
	\begin{equation*}
		\cF_{\gamma} :=\left\{ 
	\begin{aligned}
	F: &\R\rightarrow[0,1]\text{ non decreasing },\\
	& F({v})=0 \text{ when } {v}<-1,\\
	& F({v})=1\text{ when } {v}>1,\\
	& \frac{F({v})}{\gamma}\in N_{\gamma,1} \text{ when } \pm{\frac{{v}}{\gamma}} \in N_{\gamma, 1}, \\
	&F({v}) = F(\lfloor{{v}}\rfloor_{\gamma})+\frac1{\gamma}(F(\lfloor{{v}}\rfloor_{\gamma}+\gamma) -F(\lfloor{{v}}\rfloor_{\gamma}))({v}-\lfloor{{v}}\rfloor_{\gamma})\text{ otherwise}
 	\end{aligned}
 \right\}.
 	\label{equ_cdf_set}
	\end{equation*}
	\label{def_cdf_set}
\end{definition}
Definition \ref{def_parameter_set} is straightforward as we use $\Delta^d$-grids to discretize the $[0,1]^d$ space. Definition \ref{def_cdf_set} actually represents such a family of CDF: the random variable is defined on $[-1,1]$, and its CDF equals some integer multiple of $\gamma$ when ${v}$ (or $-{v}$) itself is an integer multiple of $\gamma$; for those ${v}$ in between these grids, CDF connects the two endpoints as linear. In a word, each CDF in $\cF_{\gamma}$ is a piecewise linear function with every integer-multiple-$\gamma$ points valuating some integer-multiple-$\gamma$ as well. From the definitions above, we know that $|\Omega_{\Delta, d}|=O\left((\frac{1}{\Delta})^d\right)$. Also, we have $|\cF_{\gamma}|=\binom{\frac3{\gamma}}{\frac1{\gamma}} =O(2^{\frac3{\gamma}})$ according to a ``balls into bins'' model in combinatorial counting: At each point $\frac{\pm i}{\gamma}$ (for $i\in[\frac2\gamma]$) the CDF can increase by $j\cdot\gamma$, with $j$ being a non-negative integer, and the summation of all increases is $1$ (i.e., $\frac1\gamma$ of $\gamma$ increments).

Finally we introduce the \emph{EXP-4} algorithm \citep{auer2002nonstochastic} for adversarial contextual bandits. With a finite action set $A$ and policy set $\Pi$, the EXP-4 agent has a regret guarantee at $O(\sqrt{T|A|\log{|\Pi|}})$ in $T$ rounds (comparing with the optimal policy in $\Pi$). The following is a simplified version of EXP-4 that illustrates its mechanism. For a more detailed introduction, please directly refer to \citet{auer2002nonstochastic}.

\fbox{\parbox{0.93\textwidth}{EXP-4.
		\small
		\noindent
		\begin{algorithmic}
			\STATE {\bfseries Input:} {Policy set $\Pi$, Action set $A$.}
			\STATE Initialize each policy $i$ with weight $w_i$;
			\FOR{$t=1$ {\bfseries to} $T$}
			\STATE Set probability $p_j(t)$ for each action $j$ according to weights of all policies;
			\STATE Get $a_t$ by Thompson sampling the action set $A$ according to current probability $\{p_j(t)\}$;
			\STATE Receive a reward $r_t$;
			\STATE Construct an \emph{Inverse Propensity Scoring (IPS)} estimator $\hat{r}_i(t)$ for the reward of each action $i$.
			\STATE Update weights $w_i$'s according to $\hat{r}_i(t)$.
			\ENDFOR
		\end{algorithmic}
	}
}

\subsection{Linear-EXP4 for LP}
\label{subsec_linear_exp4}

Here we present our ``Linear-EXP4'' algorithm for the linear policy pricing problem. It takes $\Omega_{\Delta, d}$ as the policy set and plug it into EXP-4 algorithm, which is straightforward but significant in reducing the regret. The pseudo-code of Linear-EXP4 is summarized as Algorithm \ref{algo_linear_exp4}.

\begin{algorithm}[H]
	\caption{Linear-EXP4}
	\label{algo_linear_exp4}
	
	\begin{algorithmic}
		\STATE {\bfseries Input:} {Parameter set $\Omega_{\Delta, d}$, Action set $A_{\gamma}=\{0, \gamma, 2\gamma, \ldots, \lfloor{B}\rfloor_{\gamma}\}$, parameters $\Delta, \gamma$.}
		\STATE Set policy set $\Pi^{LP}_{\Delta, \gamma} = \{\pi_{\beta}(x) = \lfloor x^{\top}\beta\rfloor_{\gamma}, \beta\in\Omega_{\Delta, d}\}$
		\STATE Initialize an EXP-4 agent $\cE_{LP}$ with $\Pi^{LP}_{\Delta, \gamma}, A_{\gamma}$;
		\FOR{$t=1$ {\bfseries to} $T$}
		\STATE $\cE_{LP}$ observe $x_t$;
		\STATE $\cE_{LP}$ choose an action (price) $v_t$;
		\STATE Receive feedback $r_t = v_t\cdot\ind_t$ and feed it into $\cE_{LP}$;
		\ENDFOR
\end{algorithmic}
\end{algorithm}

Here the EXP-4 agent $\cE_{LP}$ would approach the best policy $\pi^*$ in $\Pi^{LP}_{\Delta, \gamma}$ within a reasonable regret. Therefore, we have to carefully choose $\Delta$ and $\theta$ such that the regrets of both $\cE_{LP}$ and $\pi^*$ are well bounded.

\subsection{Discrete-Distribution-EXP4 for LV}
\label{subsec_d2_exp4}

Here we present our ``Discrete-Distribution-EXP-4'' algorithm, or D2-EXP4 for the linear noisy valuation pricing problem. Though it originates EXP-4 as well as Linear-EXP4 above, the reduction is not as straightforward. In fact, the policy set is defined as follows:
\begin{footnotesize}
\begin{equation}
\label{equ_policy_d2_exp4}
\begin{aligned}
\Pi^{LV}_{\Delta, \gamma}=&\left\{\pi|\pi(x; \hat{\theta}, \hat{F})=\max\{\lfloor{x^{\top}\hat{\theta}}\rfloor_{\gamma}-(B+1)\gamma+\lfloor{w^*(x)}\rfloor_{\gamma}, 0\}, \right.\\
 \text{ where }&\left. w^*(x) =\argmax_w g(u+w, x^{\top}\hat{\theta}, \hat{F}), \hat{\theta}\in\Omega_{\Delta, d}, \hat{F}\in\cF_{\gamma}\right\}.
\end{aligned}
\end{equation}
\end{footnotesize}
For each policy in $\Pi^{LV}_{\Delta, \gamma}$, it firstly takes a $\hat{\theta}$ from $\Omega_{\Delta, d}$ and a $\hat{F}$ from $\cF_{\gamma}$, and then generate an ``optimal incremental price'' $w^*(x)$ greedily as if they are the true parameter $\theta^*$ and the true noise distribution $F$. Finally, the policy take an action (price) that is the summation of $\gamma$-lower roundings of $\hat{u}=x^{\top}\hat{\theta}$ and $w^*(x)$ to fit in the action set $A_{\gamma}:=\{0, \gamma, 2\gamma, \ldots, \lfloor{B+1}\rfloor_{\gamma}\}$, and minus a $(B+1)\gamma$ amount. We know that $|\Pi^{LV}_{\Delta, \gamma}| = |\Omega_{\Delta, d}|\cdot|\cF_{\gamma}| = O((\frac1{\Delta})^d\cdot2^{\frac3{\gamma}})$. We present the psuedo-code of D2-EXP4 as Algorithm \ref{algo_d2_exp4}.

\begin{algorithm}[H]
	\caption{Discrete-Distribution-EXP-4(D2-EXP4)}
	\label{algo_d2_exp4}
	
	\begin{algorithmic}
		\STATE {\bfseries Input:} {Policy set $\Pi^{LV}_{\Delta, \gamma}$, Action set $A_{\gamma}=\{0, \gamma, 2\gamma, \ldots, \lfloor{B+1}\rfloor_{\gamma}\}$, parameters $\Delta, \gamma$.}
		\STATE Initialize an EXP-4 agent $\cE_{LV}$ with $\Pi^{LV}_{\Delta, \gamma}, A_{\gamma}$;
		\FOR{$t=1$ {\bfseries to} $T$}
		\STATE $\cE_{LV}$ observe $x_t$;
		\STATE $\cE_{LV}$ select an action(price) $v_t$;
		\STATE Receive feedback $r_t = v_t\cdot\ind_t$ and feed it into $\cE_{LV}$;
		\ENDFOR
	\end{algorithmic}
\end{algorithm}

D2-EXP4 is straightforward that it takes the $\gamma$-rounding of a greedy price, except the $(B+1)\gamma$ price markdown. This is because we want a conservative price, and the $(B+1)\gamma$ markdown is to compensate the ``exaggerate'' $\lceil\theta\rceil_{\gamma}$ parameter we adopt in $\Pi^{LV}_{\Delta, \gamma}$. We will include more details in Paragraph \ref{paragraph_conservative} below and in Section \ref{subsec_upper_bound}.

\paragraph{Adversarial Features and Agnostic Distributions} Notice that both algorithms are suitable for adversarial $x_t$ series, which is a property of EXP-4. It is worth mentioning that our Linear-EXP4 makes \textbf{no} assumptions on the distribution of $y_t\text{ given } x_t$, and that D2-EXP4 assumes \textbf{no} pre-knowledge or technical assumptions on the noise distribution (despite that noises are bounded).
\paragraph{Conservative Pricing Strategy}\label{paragraph_conservative} Both of our algorithms adopt a \emph{conservative} strategy while pricing: In Linear-EXP4, a good-enough linear policy is the $\gamma$-lower rounding of parameter $\beta^*$; in D2-EXP4, we even define each policy by proposing a ``greedy-and-safe'' price which takes a $(B+1)\gamma$-markdown on the output of the optimal greedy pricing policy. This is because of the ``half-Lipschitz'' nature of a demand curve: decreasing the price would at least maintain the chance of being accepted. Since we do not make any Lipschitz or smoothness assumptions on the distributions, these discretizations might marginally increase the price and cause drastic change of the expected revenue. In order to avoid this, it is always better to decrease the proposed price by an acceptable small amount as it guarantees the probability of acceptance.
\paragraph{Computational Efficiency} Our algorithms require exponential computations w.r.t. dimension $d$ since the EXP-4 agent requires exponential time to evaluate each policy in the policy set. An ``optimization oracle''-efficient contextual bandit algorithm in \citet{agarwal2014taming} can be used in place of EXP-4 to achieve a near-optimal regret (up to logarithmic factors), but it requires the input features $x_t$ to be drawn from an unknown fixed distribution. 

\section{REGRET ANALYSIS}	
\label{sec_regret_analysis}
In this section, we analyze our Linear-EXP4 and D2-EXP4 algorithm and prove their $\tilde{O}(d^{\frac13}T^{\frac23})$ and $O(T^{\frac34})$ regret bounds, respectively. Also, we present a scenario where a lower bound construction with $\tilde{\Omega}(T^{\frac23})$ regret fits for both LP and LV problems, even under stronger assumptions including stochastic $x_t$'s, Lipschitz distribution functions and unimodal demand curves.

\subsection{Upper Bounds}
\label{subsec_upper_bound}
Here we propose the following theorem as a regret bound of Linear-EXP4. This only requires the assumption that features $x_t$'s and (potential) optimal parameter $\beta^*$ is bounded by $L_2$-norm, without making any specifications on the feature-valuation mapping.

\begin{theorem}[Regret of Linear-EXP4] In any LP problem, with Assumption \ref{assump_bounded_features_parameters}, the expected regret of Linear-EXP4 does not exceed $O(d^{\frac13}T^{\frac23}\log{dT})$ by setting $\Delta = T^{-\frac13}d^{-\frac16}$ and $\gamma = T^{-\frac13}d^{\frac13}$.
	\label{theorem_upper_LP}
\end{theorem}
\begin{proof}
	We denote $\tilde{\beta}^*=\lfloor\beta^*\rfloor_{\Delta}$ and $\hat{\beta}^*:= \argmax_{\beta\in\Omega_{\Delta, d}}\sum_{t=1}^T\E[h(\pi_{\beta}(x_t), x_t)]$.
	Now we decompose the regret of LP problem as follows:
	\begin{footnotesize}
		\begin{equation}
		\begin{aligned}
		\E[Reg_{LP}]=&\sum_{t=1}^T\E[h(x_t^{\top}\beta^*, x_t)-h(v_t, x_t)]\\
		=&\sum_{t=1}^T\E[h(x_t^{\top}\beta^*, x_t)-h(\pi_{\tilde{\beta}^*}(x_t), x_t)] +\E[h(\pi_{\tilde{\beta}^*}, x_t)-h(\pi_{\hat{\beta}^*}(x_t), x_t)] +\E[h(\pi_{\hat{\beta}^*}(x_t), x_t)-h(v_t, x_t)]\\
		\leq&\sum_{t=1}^T(x_t^{\top}\beta^*-x_t^{\top}\tilde{\beta}^*)F_{LP}(x_t^{\top}\beta^*|x_t) +\E[h(\pi_{\tilde{\beta}^*}, x_t)-h(\pi_{\hat{\beta}^*}(x_t), x_t)] +\E[h(\pi_{\hat{\beta}^*}(x_t), x_t)-h(v_t, x_t)]\\
		\leq&\sum_{t=1}^TB\cdot\Delta\sqrt{d} + 0 + \sqrt{T\cdot\frac{1}{\gamma}\cdot\log{(\frac1{\Delta})^d}}\\
		=&O(d^{\frac13}T^{\frac23}\log{dT}).
		\end{aligned}
		\label{equ_LP_decomposition_regret}
		\end{equation}
	\end{footnotesize}
	Here the third row is because $\pi_{\tilde{\beta}^*}(x_t) = \lfloor{x_t^{\top}\tilde{\beta}^*}\rfloor_{\gamma}\leq x_t^{\top}\tilde{\beta}^*\leq x_t^{\top}\beta^*$ since $x_t, \beta\in\R^d_+$ (and thus $F_{LP}(x_t^{\top}\beta^*)\leq F_{LP}(x_t^{\top}\tilde{\beta}^*)$); The fourth row is because $(x_t^{\top}\beta^*-x_t^{\top}\tilde{\beta}^*)\leq\|x_t\|_2\cdot\|\beta^*-\tilde{\beta}^*\|\leq B\cdot\Delta\sqrt{d}$, the optimality definition of $\hat{\beta}^*$ and the regret bound of EXP-4 from \citet{auer2002nonstochastic}; The last row is got by plugging in the value of $\Delta$ and $\gamma$.
\end{proof}
The proof of Theorem \ref{theorem_upper_LP} is straightforward based on the existing $O(\sqrt{T|A|\log{|\Pi|}})$ bound of EXP-4. We only have to bound the error of the optimal policy in $\Pi_{\Delta, \gamma}$. Now we present our result on D2-EXP4:
\begin{theorem}[Regret of D2-EXP4]
	For any LV problem, with Assumptions \ref{assump_bounded_features_parameters}, \ref{assump_decrease_demand_LP} and \ref{assump_bounded_noise}, our algorithm D2-EXP4 guarantees a regret no more than $O(T^{\frac34}+T^{\frac23}d^{\frac12}\log{dT})$ as we set $\Delta = T^{-\frac14}d^{-\frac12}$ and $\gamma = T^{-\frac14}$.
	\label{theorem_upper_LV}
\end{theorem}
The proof of Theorem \ref{theorem_upper_LV} is more sophisticated than that of Theorem \ref{theorem_upper_LP}, but they shares similar structures: we figure out one specific policy in $\Pi^{LV}_{\Delta, \gamma}$ that is close to the optimal policy of the LV problem. The main idea of this proof is to find out a tuple of $(\hat{\theta}, \hat{F})$ that approaches the true parameter and distribution, and to verify that the policy built on this approaching tuple is reliable only within small tractable error. The highlight is that we do not assume any Lipschitzness on the distribution, which is quite different from existing approximation methods. In fact, it is the natural property of pricing problems that enables this: for two prices $v_1\geq v_2$, the probability of $v_2$ being accepted is greater (or equal) than that of $v_1$, and thus $(v_1-v_2)\geq g(v_1, u, F)-g(v_2, u, F)$. We may call it a \emph{Half-Lipschitz} property since it only upper bounds the increasing rates.

Here we show a proof sketch of Theorem \ref{theorem_upper_LV}, and leave the bulk to Appendix \ref{appendix_proof_upper_bound_of_lv}.

\begin{proof}[Proof Sketch]
	For any specific LV problem with linear parameter $\theta^*$ and noise CDF $F$, we define $\hat{\theta}^*:=\lceil\theta^*\rceil_{\Delta}$ and $\hat{F}$:
	
	\begin{equation}
	\begin{aligned}
	\hat{F}(x)
	=& \lfloor{F(x)}\rfloor_{\gamma}\text{ when } x=i\cdot\gamma\text{ for }i\in\mathbb{Z},\text{ and }\text{ linearly connecting }
	\hat{F}({i}{\gamma})\text{ with } \hat{F}({i+1}{\gamma})\\
	&\text{ when } x\in({i}{\gamma}, (i+1){\gamma}).\\
	\end{aligned}
	\label{equ_tilde_f_and_hat_f_and_hat_theta_star_main}
	\end{equation}
	Our goal is to prove that $\pi(x; \hat{\theta}^*, \hat{F})$ performs well enough. We may furthermore define a few amounts:
	\begin{enumerate}[label=(\roman*)]
		\item $\hat{u} = x^{\top}\hat{\theta}^*$;
		\item $w^*(u)=\argmax_w g(u+w, u, F)$;
		\item $ \hat{w}^*(u)=\argmax_w g(u+w, u, F)$;
		\item $ \hat{w}(\hat{u})=\argmax_w g(\hat{u}+w, \hat{u}, \hat{F})$.
	\end{enumerate}
	Therefore, the price our algorithm proposed for feature $x$ is $\hat{v}(x) = \lfloor\hat{u}\rfloor_{\gamma}-(B+1)\gamma+\lfloor\hat{w}(\hat{u})\rfloor_{\gamma}$, and our goal is to prove that $g(\hat{v}, u, F)\geq g(u+w^*(u), u, F) - C\cdot\gamma$ for some constant $C$. Since $\gamma = T^{-\frac14}$, this would upper bounds the optimality error up to $O(T\cdot\gamma)=O(T^{\frac34})$. In fact, we have the following properties:
	\begin{enumerate}[label=(\roman*)]
		\item $\hat{\theta}^*=\lceil\theta^*\rceil_{\Delta}$ (by definition);
		\item $\|\theta^*\|_2\leq\|\hat{\theta}^*\|_2\leq\|\theta^*\|_2+\Delta\sqrt{d}=\|\theta^*\|_2+\gamma$;
		\item $u-\gamma\leq\hat{u}-\gamma\leq\lfloor\hat{u}\rfloor_{\gamma}\leq\hat{u}\leq u+B\gamma$;
		\item $\hat{F}({i}{\gamma})\leq F({i}{\gamma})\leq\hat{F}({i}{\gamma})+\gamma$.
	\end{enumerate}
	According to these properties, we may derive:
	\begin{equation*}
	\begin{aligned}
	&g(\lfloor\hat{u}\rfloor_{\gamma}-(B+1)\gamma+\lfloor\hat{w}(\hat{u})\rfloor_{\gamma}, u, F)\\
	\geq&(u+\lfloor\hat{w}(\hat{u})\rfloor_{\gamma})(1-F(\lfloor\hat{w}(\hat{u})\rfloor_{\gamma}))-(B+2)\gamma\\
	\geq&(u+\lfloor\hat{w}(\hat{u})\rfloor_{\gamma})(1-\hat{F}(\hat{w}(\hat{u})))-(2B+3)\gamma\\
	\geq&g(\hat{u}+\hat{w}(\hat{u}), \hat{u}, \hat{F})-(3B+4)\gamma\\
	\geq&g(u+\hat{w}^*(u), u, \hat{F})-(3B+4)\gamma\\
	\geq&g(u+w^*(u), u, F)-(3B+5)\gamma.
	\end{aligned}
	\end{equation*}
	The derivation of each step is shown in Appendix \ref{appendix_proof_upper_bound_of_lv}. With this policy-realizability error being bounded by $(3B+5)\gamma=O(T^{\frac34})$ and the original regret of the EXP-4 agent being $O(\sqrt{TK\log N})=\tilde{O}(T^{\frac34} + d^{\frac12}T^{\frac58})$, we may finally get a $\tilde{O}(T^{\frac34}+d^{\frac12}T^{\frac58})$ upper regret bound.
\end{proof}

\subsection{Lower Bounds}
\label{subsec_lower_bound}

In this part, we present an $\tilde{\Omega}(T^{\frac23}d^{\frac13})$ and an $\tilde{\Omega}(T^{\frac23})$ regret lower bounds that hold for LP and LV problems respectively. We will firstly claim a lower bound for non-contextual pricing problem, and then generalize the result to LP and LV.

\begin{theorem}[Lower bound for non-contextual pricing]\label{theorem_lower_bound_main}
	For a non-contextual pricing problem where the valuation $y_t$'s are generated independently and identically from a fixed unknown distribution satisfying (1) the CDF $F(y)$ is Lipschitz and (2) the revenue curve $g(v, F) = y\cdot(1-F(v))$ is unimodal (i.e., non-decreasing on $(0,v_0)$ and non-increasing on $(v_0, +\infty)$ for some $v_0$), NO algorithm can achieve $O(T^{\frac23-\delta})$ for any $\delta>0$.
\end{theorem}

The detailed proof of Theorem \ref{theorem_lower_bound_main} is in Appendix \ref{appendix_proof_lower_bound}, and in the main pages we briefly demonstrate the constructions of the subproblem family where we achieve this lower bound.

Here we take the idea of \citet{kleinberg2004nearly} where they make use of bump functions and nested intervals to ensure Lipschitz continuity and unimodality, sequentially. Since that their model is not capturing a revenue curve and that their feedback is numerical instead of Boolean, we have to adjust their design to satisfy the pricing setting. On the one hand, the probability of a price to be accepted, i.e., the rate $\frac{\E[r(v)]}{v}$, is non-increasing as the prices increases, which is not guaranteed for that of a reward function of a continuum bandit (if we treat $v$ as an action). In this proof, we adopt a series of transformations to  convert the ``bump function tower'' into a revenue curve while keeping all monotonically-increasing/decreasing intervals unchanged. On the other hand, we still use the KL-divergence to distinguish among distributions, but in a different way. As for Boolean feedback, we only need to calculate the KL-divergence of two Bernoulli random variables, which can be upper bounded by a quadratic term of their probabilistic difference. 

\begin{figure}[ht]
	\centering
	\includegraphics[width=0.49\textwidth]{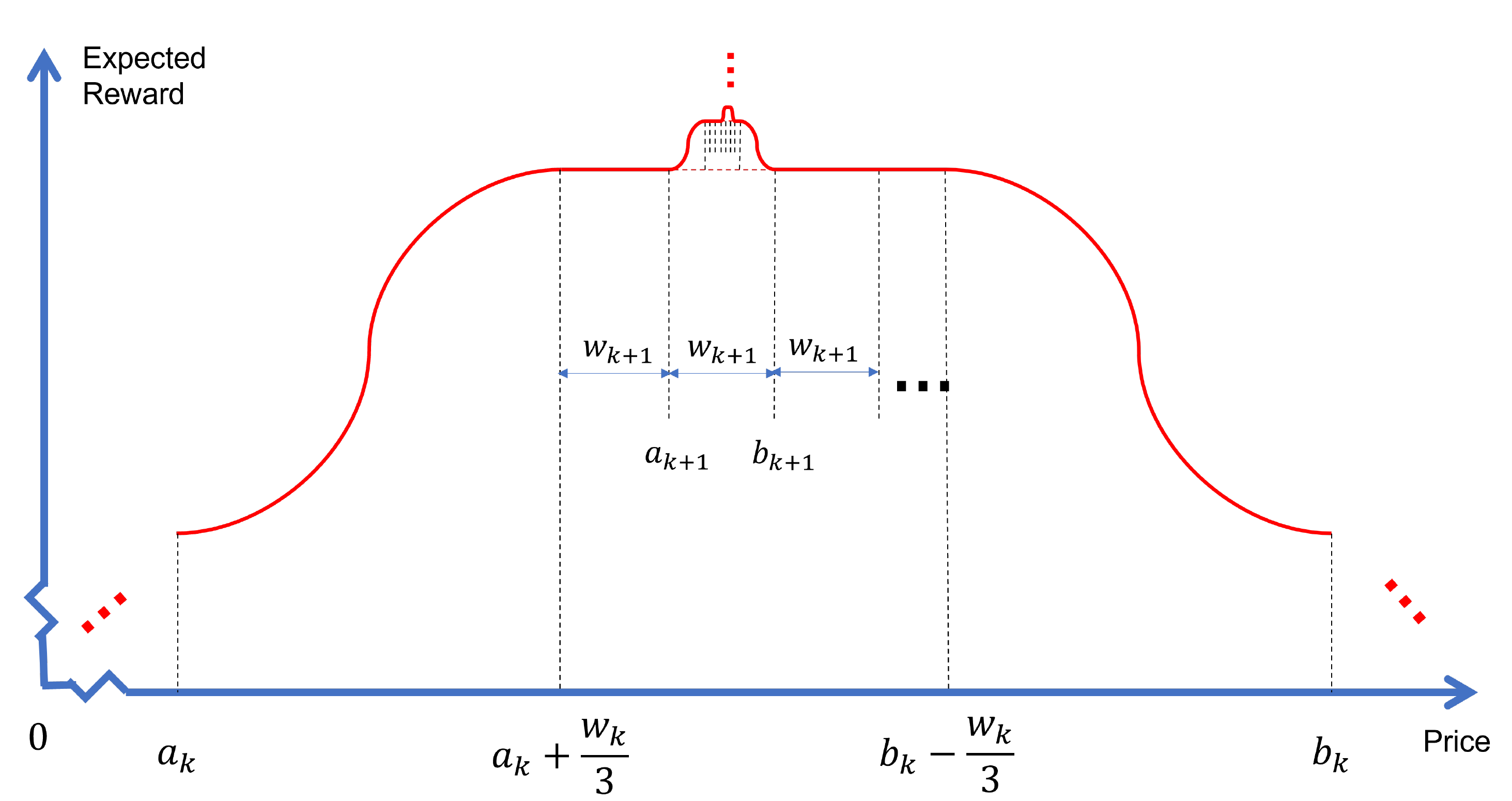}
	\includegraphics[width=0.49\textwidth]{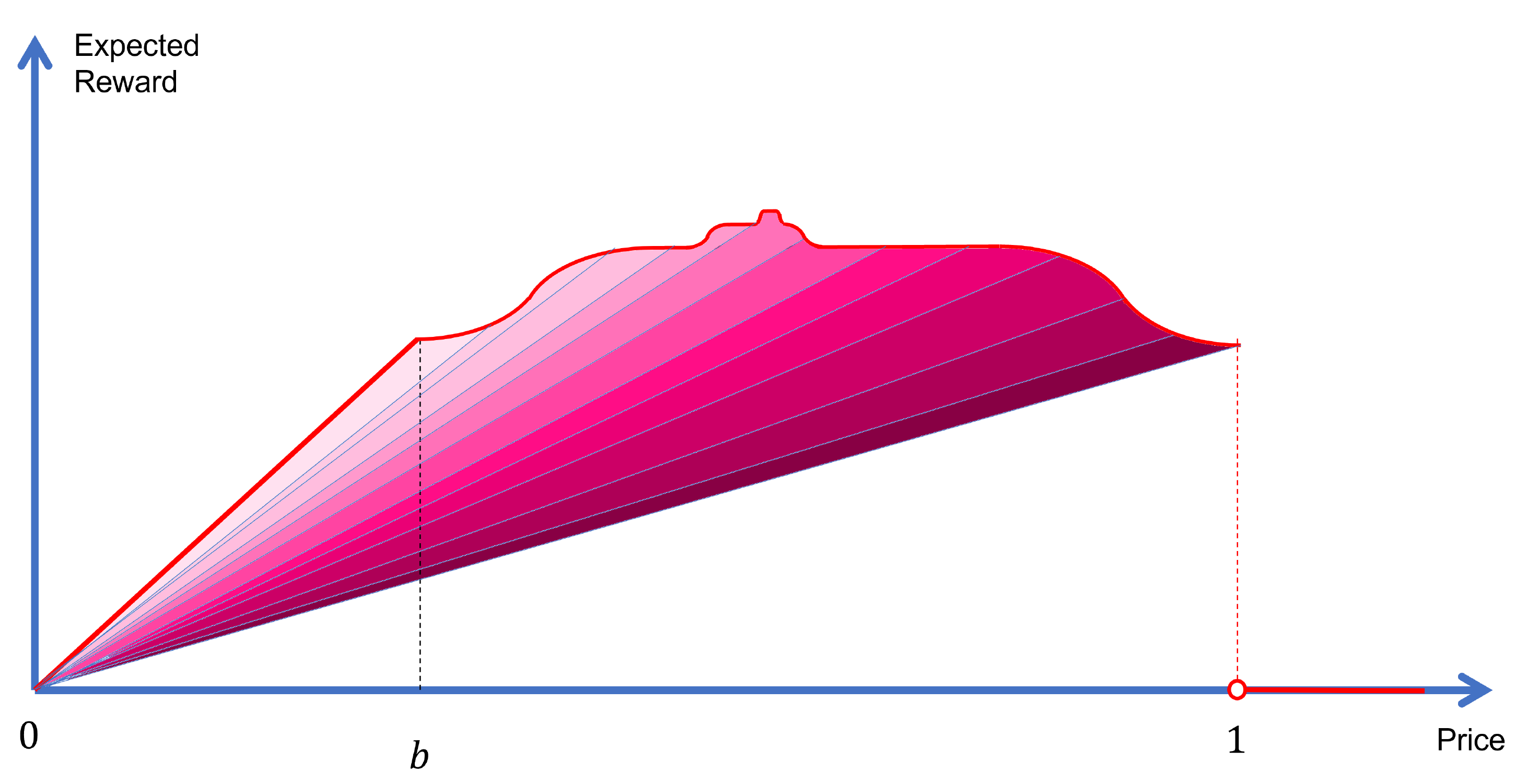}
	\caption{Structure of our lower bound function family. The left figure shows how we use bump functions to construct a reward function $f(v)$. Each bump function locates at $[a_k, b_k]$ with a length $w_k=3^{-k!}$. Notice that the middle one-third of each bump is a plain divided into small intervals of length $w_{k+1}$, and we might randomly choose one to build up the $(k+1)^{\text{th}}$ bump. However, 
		the rate $\frac{f(v)}{v}$ that indicates the probability of $v$ to be accepted is not necessarily non-increasing, and therefore $f(v)$ cannot capture a revenue function for pricing. The right figure shows an ideal revenue curve $D(v)$ which equals $v$ for $v\in[0,b]$ and equals $b+(1-b)(1-\frac1{f(v)+1})$ for $v\in(b,1]$. The slopes indicate that $\frac{D(v)}{v}$ is actually non-increasing. We draw the figures with exaggeration to show the hierarchical structures better.} 
	\label{figure_lower_bound_bump}
\end{figure}

The constructions of bump-based revenue curves are illustrated in Figure \ref{figure_lower_bound_bump}. Firstly, we define a nested-interval series $[0,1]=[a_0,b_0]\supset[a_1,b_1]\supset\ldots\supset[a_{k}, b_{k}]\supset\ldots$, where $b_k=a_k+w_k$, $w_{k}=3^{-k!}$. We let $a_k$ be chosen from the discrete set $\{a_{k-1}+\frac{w_{k-1}}3+i\cdot{w_k}, i=0,1,2,\ldots, \frac{w_{k-1}}{3w_k}\}$. Secondly, we construct Lipschitz bump functions in each $[a_k, b_k]$ interval, the middle one-third of which is a plain line 
Thirdly, we add all these bump function up, which forms a ``tower'' with its peak randomly generated by the series of tightening intervals $\{[a_k, b_k]\}$. Finally, it is transformed into a revenue curve after a series of operations.

If we treat this randomly-generated function a uniformly-distrbuted family of functions, then we can further prove our lower bound: On the one hand, we prove that the feedback cannot accurately locate where the ``peak of the tower'' is, from the perspective of information theory. In fact, any algorithm would have a constant chance of missing the peak. On the other hand, the cost of missing a peak can be lower bounded, and thus the expected regret is as well lower bounded by their product.


With this theorem holds, we can soon get the following two corollaries:
\begin{corollary}[Lower bound of LP problem]
	The regret lower bound for LP problems is $\tilde{\Omega}(d^{\frac13}T^{\frac23})$, even with stochastic features and distributional properties same as those in Theorem \ref{theorem_lower_bound_main}.
	\label{corollary_lower_bound_lp}
\end{corollary}
\begin{proof}
	Here we construct the following LP problem: let $x_t=[0, \ldots, 0, 1, 0, \ldots, 0]^{\top}$ with only the $i_t^{\text{th}}$ element being $1$, where $i_t$ is chosen from $\{1,2,\ldots, d\}$ uniformly at random for each $t=1,2,...,T$. As a result, the problem is split into $d$-subproblems with each of them a non-feature pricing problem in $\frac{T}{d}$ rounds in expectation (since the demand function $F_{LP}(y|x)$ can be totally different and independent for different $x$'s). According to Theorem \ref{theorem_lower_bound_main}, the lower bound for this problem is $\tilde{\Omega}(d\cdot(\frac{T}{d})^{\frac23})=\tilde{\Omega}(d^{\frac13}T^{\frac23})$.
\end{proof}
\begin{corollary}[Lower bound of LV problem]
	The regret lower bound for LV problems is $\tilde{\Omega}(T^{\frac23})$, even with stochastic features and noise-distributional properties stated in Theorem \ref{theorem_lower_bound_main}.
	\label{corollary_lower_bound_lv}
\end{corollary}
It is worth mentioning that the noise distribution is itself an (inversed) demand function on $(v-u)$, i.e., it is non-increasing as $(v-u)$ gets larger. Based on this insight, the derivation of Corollary \ref{corollary_lower_bound_lv} is straightforward: any non-feature pricing problem with bounded i.i.d. $y_t$'s can be reduced to an LV problem up to constant coefficients. In fact, suppose $y_t\in[a,b], 0\leq a<b$ in a non-feature pricing problem, and then we might define an LV problem by setting $d=1, \theta^*=\frac{a+b}{b-a}\text{ and } x_t=1, \forall t\in\mathbb{Z}_+$ since now $x_t^{\top}\theta^*+N_t\in[a,b]$. As long as the definition of LV problem does not specify the distributional properties (besides being bounded), the distribution family in the proof of Theorem \ref{theorem_lower_bound_main} can be reduced to an LV problem as well. In this way, the $\tilde{\Omega}(T^{\frac23})$ lower bounds are applicable to LV problems.

\section{NUMERICAL EXPERIMENTS}
\label{sec_num_experiments}
In this section, we conduct numerical experiments to show the validity of Linear-EXP4. We assume $d=2, B=1$ as basic parameters, and assume a Gaussian noisy valuation model i.e., $y_t = u_t + N_t$ where $N_t\sim\cN(0, \frac1{16})$ independently for all $t$. For the convenience of comparing with a fixed optimal linear policy $\beta^*$, we let $u_t = J^{-1}(x_t^{\top}\beta^*)$ for each $t$, where $J(u)=\argmax_v g(v,u, 1-\Phi_{\cN(0, \frac1{16})})$ is a greedy pricing function defined in \citet{xu2021logarithmic}\footnote{They also show the existence of $J^{-1}(v)$ by showing that $J'(u)\in(0,1)$.}. In other words, the linear price $v_t^*=x_t^{\top}\beta^*$ always maximizes the expected reward for any $t$, and we may calculate the empirical \emph{ex ante} regret (i.e., comparing the empirical performance with the maximizer of expected regret at each round) by comparing $v_t\cdot\ind(v_t\leq y_t)$ with $x_t^{\top}\beta^*\cdot\ind(x_t^{\top}\beta^*\leq y_t)$. According to Hoeffding's Inequality, the \emph{ex post} regret that we adopt for the LP problem is only $\tilde{O}(\sqrt{T})$ different from the empirical \emph{ex ante} regret. Given that the regret rate of Linear-EXP4 is $\tilde{\Theta}(T^{\frac23})$, we may ignore this difference and only show the ex ante regret in our experiments. Since the EXP-4 learner requires pre-knowledge on $T$ and is not an any-time algorithm (i.e., the cumulative regret is meaningful only at $t=T$), we execute Linear-EXP4 for a series of $T=\lfloor2^{\frac k3}\rfloor$ for $k=27, 28, \ldots, 48$. We repeat every experiment 20 times for each setting and then take an average. The results are shown in Figure \ref{fig_plotting}

\begin{figure}
	\centering
	\includegraphics[width=0.6\textwidth]{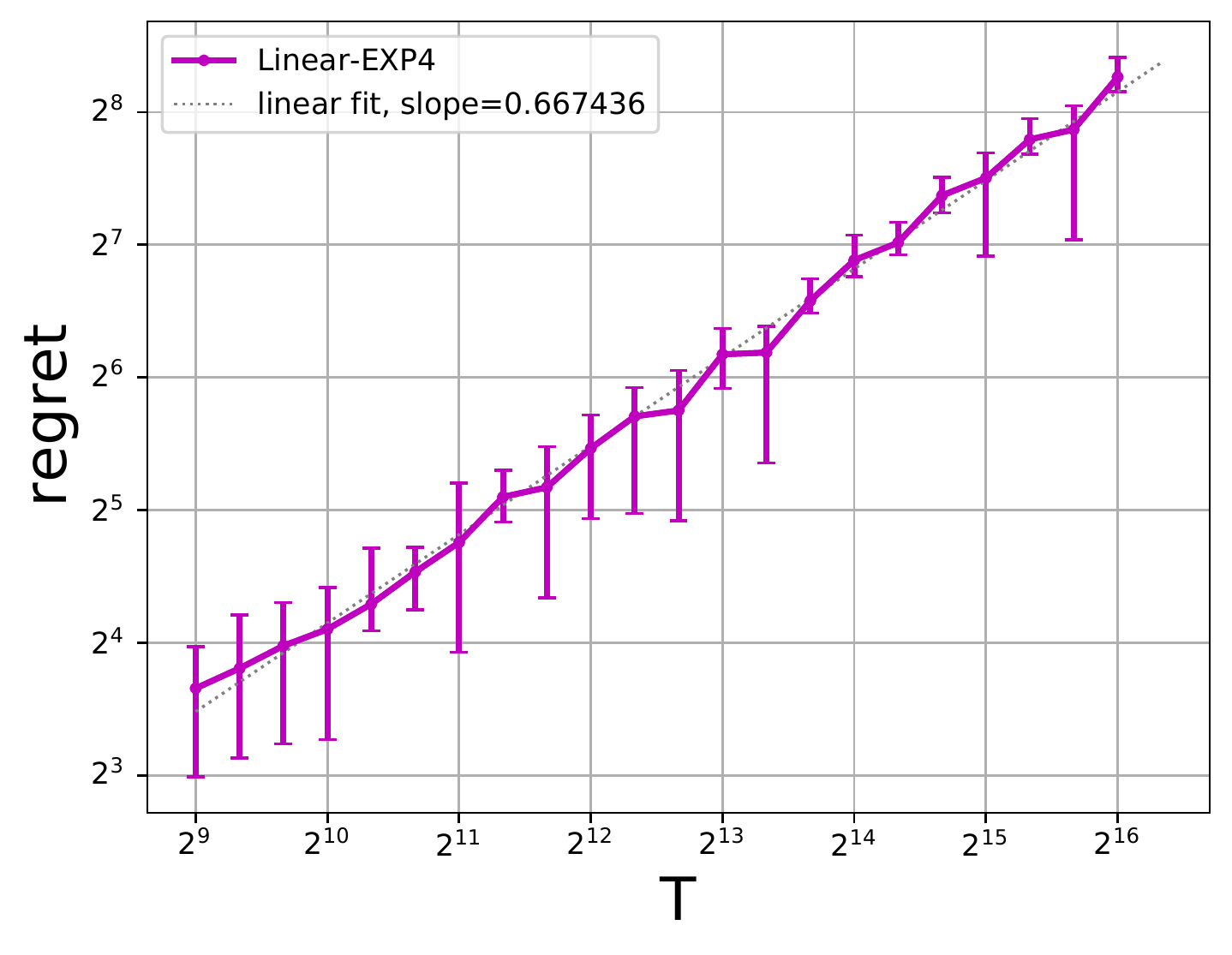}
	\caption{Regrets of Linear-EXP4 on simulated examples. The plot is on log-log scales to show the regret rate: a slope of $\alpha$ indicates an $O(T^{\alpha})$ regret. Besides, we draw error bars with 0.95 coverage. Notice that the slope of its linear fit is $0.667$, which matches the $\tilde{O}(T^{\frac23})$ regret rate in theory.}
	\label{fig_plotting}
\end{figure}


We were unable to conduct numerical experiments on D2-EXP4 due to the exponential time complexity of the EXP-4 learner along with the $2^{T^{\frac14}}$-size policy set. We provide the code of D2-EXP4 in our supplementary materials.

\section{DISCUSSION}
\label{sec_discussion}
In this section, we discuss potential extensions of this work and our conjectures on the regret of LV problems.

\paragraph{From Linear to Non-Linear}
	Both LP and LV problems are based on a linear principle of feature-price/valuation relationships, which is not reasonable in many real-world situations (for example, the price of a diamond). Based on our specifications on LP and LV problems, we may similarly define two corresponding problems: (1) We make no assumptions on the $x_t\rightarrow y_t$ mapping, but compare with the optimal policy in a parametric non-linear model space. (2) We directly assume that the $x_t\rightarrow y_t$ is a parametric non-linear function adding some unknown (and non-parametric) noise, and compare with the optimal price. We may slightly modify our Linear-EXP4 and D2-EXP4 to deal with these two problems by just replacing the linear discretized policy set with another non-linear one. However, we should be careful about any discretization involved: the $\gamma$-roundings of non-linear policy parameters do not necessarily lead to a slightly lower price (maybe either higher or much lower). Like what we designed in D2-EXP4, we still have to ensure the parametric optimal policy itself performs within a $[-O(\gamma), 0]$ range from the global optimal policy. 
\paragraph{The Minimax Regret(s) of LV} 
Existing works on solving LV have achieved various regret bounds with different assumptions. This is quite different from the linear regression problem where noise distributions do not significantly affect the result. To the best of our knowledge, we are the first to get rid of all assumptions (despite bounded-noise assumption \footnote{If the noise is neither bounded nor parametrized, then any finite-time algorithm will suffer a linear regret when the noise is very large and prices are always being accepted.}). However, we did not close the regret gap in this setting. This problem is similar to a non-feature pricing problem as we adopt the same lower bound proof in this work, but the situations are entirely different: In non-feature pricing, we aim at a fixed optimal price, and we only have to know the valuation distribution around the optimal price. However, in an LV problem, we have to approach the exact linear valuation adding an optimal \emph{increment} for each feature, and the optimal increments are \textbf{not} fixed for different valuations. As a result, we have to know the whole noise distribution. This drastically increases the hardness of LV, and we conjecture LV with a $\Theta(T^{\alpha})$ regret where $\alpha>\frac23$.  

\paragraph{Dependence on Noise Scale $R$} In this work we assume the noise $N_t\in[-1,1]$. Based on this assumption, we construct a discrete noise CDF family $\cF_{\gamma}$ whose size is $\binom{\frac3\gamma}{\frac1\gamma}$. When it changes to $N_t\in[-R, R]$ for larger $R$, the number of discrete CDF is $\binom{\frac{2R+1}\gamma}{\frac1\gamma}\leq (\frac{2R+1}\gamma)^{\frac1\gamma}$. Also, this would increase the upper bound of prices from $(B+1)$ to $(B+R)$, which would increase the number of actions by $\frac R{\gamma}$. Recall that the regret of EXP-4 is $O(\sqrt{KT\log{N}})$ where $K$ is the number of actions and $N$ is the number of policies (i.e., \# discrete $\theta$ times \# discrete CDF). Therefore, the dependence on $R$ is $O(\sqrt{R\log{R}})$.
\section{CONCLUSION}
\label{sec_conclusion}
In this work, we have studied two agnostic feature-based dynamic pricing problems: a linear pricing policy (LP) problem with no assumptions on feature-valuation mappings, and a linear noisy valuation (LV) problem with agnostic noise distributions. For the LP problem, we have presented a \emph{Linear-EXP4} algorithm whose $\tilde{O}(T^{\frac23}d^{\frac13})$ regret matches the $\tilde{\Omega}(T^{\frac23}d^{\frac13})$ lower bound up to logarithmic factors. For the LV problem, we have proposed an $\tilde{O}(T^{\frac34})$-regret algorithm \emph{D2-EXP4} along with an $\tilde{\Omega}(T^{\frac23})$ lower bound proof even with stochastic, Lipschitz and unimodal assumptions, and both of them substantially improve existing results from $O(T^{\frac23\cup(1-\alpha)})$ (with smoothness assumptions and indeterministic $\alpha$) and $\Omega(T^{\frac35})$ respectively. Both Linear-EXP4 and D2-EXP4 allow adversarial features. Besides, we have discussed the prospective generalization of this work and development of future research in feature-based dynamic pricing.


\subsubsection*{Acknowledgements}
The work is partially supported by the Adobe Data Science Award and a start-up grant from the UCSB Department of Computer Science.


\bibliography{ref_log}

\appendix

\onecolumn
{\Huge Appendix}

\section{Proof of Regret of D2-EXP4: Theorem \ref{theorem_upper_LV}}
\label{appendix_proof_upper_bound_of_lv}

\begin{proof}
	For any specific LV problem that is defined with linear parameter $\theta^*$ and noise CDF $F$, we define another parameter $\hat{\theta}^*:=\lceil\theta^*\rceil_{\Delta}$ and another CDF functions $\hat{F}$:

	\begin{equation*}
	\begin{aligned}
	\hat{F}(x)=& \lfloor{F(x)}\rfloor_{\gamma}\text{ when } x=i\cdot\gamma\text{ for }i\in\mathbb{Z},\text{ and }\text{ linearly connecting }\\
	&\hat{F}({i}{\gamma})\text{ with } \hat{F}({i+1}{\gamma})\text{ when } x\in({i}{\gamma}, (i+1){\gamma}).
	\end{aligned}
	\label{equ_tilde_f_and_hat_f_and_hat_theta_star}
	\end{equation*}
	
	Notice that $\hat{F}\in\cF_{\gamma}, \hat{\theta}^*\in\Omega_{\Delta, d}$, and our goal is to prove that $\pi(x; \hat{\theta}^*, \hat{F})$ is good enough to mach the regret. With these two definitions, we might furthermore define a few amounts: $\hat{u} = x^{\top}\hat{\theta}^*, w^*(u)=\argmax_w g(u+w, u, F), \hat{w}^*(u)=\argmax_w g(u+w, u, F), \hat{w}(\hat{u})=\argmax_w g(\hat{u}+w, \hat{u}, \hat{F})$. Therefore, the price our algorithm proposed for feature $x$ is $\hat{v}(x) = \lfloor\hat{u}\rfloor_{\gamma}-(B+1)\gamma+\lfloor\hat{w}(\hat{u})\rfloor_{\gamma}$, and our goal is to prove that $g(\hat{v}, u, F)\geq g(u+w^*(u), u, F) - C\cdot\gamma$ for some constant $C$. 
	Since $\hat{\theta}^*:=\lceil\theta^*\rceil_{\Delta}$, we have $\|\theta^*\|_2\leq\|\hat{\theta}^*\|_2\leq\|\theta^*\|_2+\Delta\sqrt{d}=\|\theta^*\|_2+\gamma$ and thus $u-\gamma\leq\hat{u}-\gamma\leq\lfloor\hat{u}\rfloor_{\gamma}\leq\hat{u}\leq u+B\gamma$. Based on this, we may get rid of $\lfloor\hat{u}\rfloor_{\gamma}$ as follows:
	\begin{small}
		\begin{equation*}
		\begin{aligned}
		&g(\lfloor\hat{u}\rfloor_{\gamma}-(B+1)\gamma+\lfloor\hat{w}(\hat{u})\rfloor_{\gamma}, u, F)\\
		=&(\lfloor\hat{u}\rfloor_{\gamma}-(B+1)\gamma+\lfloor\hat{w}(\hat{u})\rfloor_{\gamma})\cdot(1-F(\lfloor\hat{u}\rfloor_{\gamma}-(B+1)\gamma+\lfloor\hat{w}(\hat{u})\rfloor_{\gamma}-u))\\
		\geq&(\lfloor\hat{u}\rfloor_{\gamma}+\lfloor\hat{w}(\hat{u})\rfloor_{\gamma})\cdot(1-F(\lfloor\hat{u}\rfloor_{\gamma}-(u+(B+1)\gamma)+\lfloor\hat{w}(\hat{u})\rfloor_{\gamma}))-(B+1)\gamma\\
		\geq&(u+\lfloor\hat{w}(\hat{u})\rfloor_{\gamma})(1-F(\lfloor\hat{w}(\hat{u})\rfloor_{\gamma}))-(B+2)\gamma.\\
		\end{aligned}
		\end{equation*}
	\end{small}
	
	Now we target at $\lfloor\hat{w}(\hat{u})\rfloor_{\gamma}$ that occurs in both of the price term and the probability term, and we will get rid of it by two steps. Since $\hat{F}({i}{\gamma})\leq F({i}{\gamma})\leq\hat{F}({i}{\gamma})+\gamma$, we have the first step like:
	\begin{equation*}
	\begin{aligned}
	&(u+\lfloor\hat{w}(\hat{u})\rfloor_{\gamma})(1-F(\lfloor\hat{w}(\hat{u})\rfloor_{\gamma}))\\
	\geq&(u+\lfloor\hat{w}(\hat{u})\rfloor_{\gamma})(1-\hat{F}(\lfloor\hat{w}(\hat{u})\rfloor_{\gamma})-\gamma)\\
	\geq&(u+\lfloor\hat{w}(\hat{u})\rfloor_{\gamma})(1-\hat{F}(\lfloor\hat{w}(\hat{u})\rfloor_{\gamma}))-(B+1)\gamma\\
	\geq&(u+\lfloor\hat{w}(\hat{u})\rfloor_{\gamma})(1-\hat{F}(\hat{w}(\hat{u})))-(B+1)\gamma.\\
	\end{aligned}
	\end{equation*}
	Here the second inequality comes from the $(B+1)$ natural bound of any price. Again, we apply $u\geq\hat{u}-B\gamma$ and get the second step:
	
	\begin{equation*}
	\begin{aligned}
	&(u+\lfloor\hat{w}(\hat{u})\rfloor_{\gamma})(1-\hat{F}(\hat{w}(\hat{u})))\\
	\geq &(\hat{u}-B\gamma+\hat{w}(\hat{u})-\gamma)(1-\hat{F}(\hat{w}(\hat{u})))\\
	\geq &(\hat{u}+\hat{w}(\hat{u}))(1-\hat{F}(\hat{w}(\hat{u})))-(B+1)\gamma\\
	=&g(\hat{u}+\hat{w}(\hat{u}), \hat{u}, \hat{F})-(B+1)\gamma.
	\end{aligned}
	\end{equation*}
	
	Now, there are only $\hat{\cdot}$'s instead of $\gamma-$roundings, and we will get rid of those $\hat{\cdot}$'s within some $C\cdot\gamma$ errors. According to the definition of $\hat{w}(\hat{u})$ that it optimizes $g(\hat{u}+w, \hat{u}, \hat{F})$, we further have:
	\begin{equation*}
	\begin{aligned}
	&g(\hat{u}+\hat{w}(\hat{u}), \hat{u}, \hat{F})\\
	\geq&g(\hat{u}+\hat{w}^*(u), \hat{u}, \hat{F})\\
	=&(\hat{u}+\hat{w}^*(u))(1-\hat{F}(\hat{w}^*(u)))\\
	\geq&(u+\hat{w}^*(u))(1-\hat{F}(\hat{w}^*(u)))\\
	=&g(u+\hat{w}^*(u), u, \hat{F})
	\end{aligned}
	\end{equation*}
	Finally, according to the definition of $\hat{w}^*(u)$ that it optimizes $g(u+w, u, \hat{F})$, we have:
	\begin{equation*}
	\begin{aligned}
	&g(u+\hat{w}^*(u), u, \hat{F})\\
	\geq&g(u+\lfloor{w}^*(u)\rfloor_{\gamma}, u, \hat{F})\\
	=&(u+\lfloor{w}^*(u)\rfloor_{\gamma})(1-\hat{F}(\lfloor{w}^*(u)\rfloor_{\gamma}))\\
	\geq&(u+\lfloor{w}^*(u)\rfloor_{\gamma})(1-F(\lfloor{w}^*(u)\rfloor_{\gamma}))\\
	\geq&(u+w^*(u)-\gamma)(1-F(w^*(u)))\\
	\geq&g(u+w^*(u), u, F)-\gamma.
	\end{aligned}
	\end{equation*}
	Here the fourth line is again due to $\hat{F}({i}{\gamma})\leq F({i}{\gamma})\leq\hat{F}({i}{\gamma})+\gamma$ and the non-decreasing property of $F$. We make a tricky use of $\lfloor\cdot\rfloor_{\gamma}$ as a ``ladder'' helping us climb between $F$ and $\hat{F}$, and the ladders only emerge on those $i\gamma$ places as $i\in\mathbb{Z}$. Therefore, we have $g(\hat{v}, u, F)\geq g(u+w^*(u), u, F) - (3B+5)\cdot\gamma$.
	Since $\gamma = T^{-\frac14}$, this would upper bounds the optimality error up to $O(T\cdot\gamma)=O(T^{\frac34})$. Also, the EXP-4 agent would cause a regret of $O(\sqrt{T|A|\log{\Pi^{LV}_{\Delta, \gamma}}}) = O(\sqrt{\frac{T}{\gamma^2}+\frac{Td\log{dT}}{\gamma}})=O(T^{\frac34}+T^{\frac58}d^{\frac12}\log{dT})$. This completes the proof. 
\end{proof}

\section{Proof of Lower Bound: Theorem \ref{theorem_lower_bound_main}}
\label{appendix_proof_lower_bound}
Before the proof begins, we make some necessary definitions.
First of all, define a \emph{bump function} as following:
\begin{definition}[Bump function]
	For ${v}\in\R+$, we define
	\begin{equation*}
	B({v})=\left\{
	\begin{array}{lcl}
	0 &   & {v}\in(-\infty,0]\cup[1,+\infty)\\
	\exp\{\frac{1}{(3{v}-1)^2-1}\} \qquad\quad &   & {v}\in (0,1/3) \\
	1 &   & {v}\in[1/3, 2/3]\\
	\exp\{\frac{1}{(3{v}-2)^2-1}\}&  & {v}\in(2/3,1)\\
	0 &   &{v}\in[1,+\infty)
	\end{array}
	\right.
	\end{equation*}
	as a basic bump function. Then we define a \emph{rescaled bump function}:
	\begin{equation*}
	B_{[a,b]}({v})=B(\frac{{v}-a}{b-a}).
	\end{equation*}
	\label{def_bump_fc}
\end{definition}
Here we present a lemma on the Lipschitzness of $B({v})$:
\begin{lemma}[Lipschitz continuity of $B({v})$]
	$B({v})$ is 6-Lipschitz, i.e., $|B'({v})|\leq 6$. Also, $|B'_{[a,b]}({v})|=|\frac{1}{b-a}B'(\frac{{v}-a}{b-a})|\leq\frac{6}{b-a}$.
	\label{lemma_lip_B}
\end{lemma}
\begin{proof}
	According to Definition \ref{def_bump_fc}, we have:
	\begin{equation*}
	B'({v})=\left\{
	\begin{array}{lcl}
	0 &   & {v}\in(-\infty,0]\\
	-\frac{1}{((3{v}-1)^2-1)^2}\cdot6(3{v}-1)\exp\{\frac{1}{(3{v}-1)^2-1}\} \qquad\quad &   & {v}\in (0,1/3) \\
	0 &   & {v}\in[1/3, 2/3]\\
	-\frac{1}{((3{v}-2)^2-1)^2}\cdot6(3{v}-2)\exp\{\frac{1}{(3{v}-2)^2-1}\}&  & {v}\in(2/3,1)\\
	0 &   &{v}\in[1,+\infty)
	\end{array}
	\right.
	\end{equation*}
	Now we propose a lemma:
	\begin{lemma}
		For $t > 1$, we have $\frac{t^2}{e^t}\leq 1$.
	\end{lemma}
	In fact, for both $1<t\leq\sqrt{e}$ and $t\geq 2$, the inequality is trivial. For $t\in(\sqrt{e},2)$, we have $ln(e^t)>\ln(e^{\sqrt{e}})=\sqrt{e}\cdot1>1.6>2\times 0.7>2\ln{2}>2\ln{t}=\ln(t^2)$.
	
	Now we denote $t_1 = -\frac{1}{(3{v}-1)^2-1}, t_2 = -\frac{1}{(3{v}-2)^2-1}$, and we know that $t_1>1$ for ${v}\in(0,\frac13)$ and $t_2>1$ for ${v}\in(\frac23,1)$
	\begin{equation*}
	B'({v})=\left\{
	\begin{array}{lcl}
	0 &   & {v}\in(-\infty,0]\\
	-t_1^2\cdot6(3{v}-1)\exp\{-t_1\} \qquad\quad &   & {v}\in (0,1/3) \\
	0 &   & {v}\in[1/3, 2/3]\\
	-t_2^2\cdot6(3{v}-2)\exp\{-t_2\}&  & {v}\in(2/3,1)\\
	0 &   &{v}\in[1,+\infty)
	\end{array}
	\right.
	\end{equation*}
	Given the lemma above, we can immediately see that $-6\leq B'({v})\leq6$. This ends the proof of Lemma \ref{lemma_lip_B}.
\end{proof}

Secondly, we define a series of intervals $[0,1]=[a_0,b_0]\supset[a_1,b_1]\supset\ldots\supset[a_{k}, b_{k}]\supset\ldots$, where $b_k=a_k+w_k$, $w_{k}=3^{-k!}$. Notice that $w_{k}$ shrinks even faster than exponential series. Now we describe how to choose $[a_k, b_k]$ from $[a_{k-1}, b_{k-1}]$: We divide the range $[a_{k-1}+\frac{w_{k-1}}{3}, b_{k-1}+\frac{w_{k-1}}{3}]$ into $Q_k=\frac{w_{k-1}}{3w_{k}}$ sub-intervals of the same length $w_k$, and then we pick one of these sub-intervals uniformly at random and denote it as $[a_k, b_k]$. It is trivial to see that $[a_1, b_1]=[\frac13,\frac23]$, $[a_2, b_2]=[\frac49, \frac59]$.

Thirdly, we define a function:
\begin{equation}
f({v}):=C_f\cdot\sum_{k=0}^{\infty}w_k\cdot B_{[a_k, b_k]}({v}),
\label{def_eqn_f}
\end{equation}
where $C_f>0$ is a constant which we will determine later. There are a few properties of $f({v})$ shown in the following lemma:
\begin{lemma}
	Define $f({v})$ as Equation \ref{def_eqn_f}, and we have:
	\begin{enumerate}
		\item There exists a unique ${v}^{*}\in[0,1]$ such that $f({v}^{*})=\max_{{v}\in[0,1]}f({v})$. In specific, ${v}^{*}=\Cap_{k=1}^{\infty}[a_k, b_k]$.
		\item $f({v})$ is unimodal.
		\item For any ${v}\in[0,1]$, there exists at most one $k$, such that $B'_{[a_k,b_k]}({v})\neq 0$.
		\item $f({v})\leq\frac32C_f$.
	\end{enumerate}
	\label{lemma_f_property}
\end{lemma}
\begin{proof}
	To prove 1, we first see that ${v}^{*}\in[a_k, b_k], k=1,2,\ldots$. Notice that $\lim_{k\rightarrow\infty}a_k$ exists (since $\{a_k\}_{k=0}^{\infty}$ is increasing and upper bounded) and that $\lim_{k\rightarrow\infty}(b_k-a_k)=\lim_{k\rightarrow\infty}3^{-k!}=0$. Therefore, $\Cap_{k=1}^{\infty}[a_k, b_k]$ is a unique real number within $[\frac{1}{3},\frac{2}{3}]$.
	
	To prove 2, notice that every $B_{[a_k, b_k]}({v})$ is non-decreasing in $[0,{v}^{*}]$ and non-increasing in $[{v}^{*}, 1]$.
	
	To prove 3, consider the case when $B'_{[a_k, b_k]}({v})\neq0$, and we know that: (1)${v}\in[a_k, b_k]\subset[a_{k-1}+\frac{w_{k-1}}{3}, b_{k-1}-\frac{w_{k-1}}3]\subset[a_{k-2}+\frac{w_{k-2}}{3}, b_{k-2}-\frac{w_{k-2}}3]\subset\ldots\subset[a_0+\frac{w_0}{3}, b_0-\frac{w_0}3]=[\frac13,\frac23]$. Since $B_{[a_j, b_j]}'({v})=0, {v}\in[a_{j}+\frac{w_{j}}{3}, b_{j}-\frac{w_j}3]$, we know that $B_{[a_j, b_j]}'({v})=0, j=0,1,\ldots, k-1$. (2) ${v}\notin[a_{k+1}, b_{k+1}]\supset[a_{k+2}, b_{k+2}]\supset\ldots$, and we know that $B_{[a_i,b_i]}'({v})=0, i=k+1, k+2, \ldots$.

	To prove 4, just notice that $B({v})\leq 1$ and thus $f({v})\leq C_f\cdot\sum_{k=0}^{\infty}3^{-k!}\leq C_f\cdot\sum_{k=0}^{\infty}3^{-k!}=\frac32C_f$.
	
\end{proof}

According to Lemma \ref{lemma_f_property} Property 3, we have:
\begin{equation*}
\begin{aligned}
|f'({v})|&\leq C_f\max_{{v}\in[a_k, b_k], k=1,2,\ldots}|w_k\cdot B'_{[a_k, b_k]}({v})|\\
&= C_f\max_{{v}\in[a_k, b_k],k=1,2,\ldots}|w_k\cdot B'(\frac{{v}-a_k}{b_k-a_k})\cdot\frac{1}{w_k}|\\
&\leq C_f\max_{y}|B'(y)|\\
&\leq 6C_f
\end{aligned}
.
\end{equation*}

This holds for any ${v}\in[0,{v}^*)\cup({v}^*, 1]$. Now we define another function $G({v})$\footnote{Here G stands for ``gain'', which is different from the revenue curve to be introduced later.}:
\begin{equation}
G({v})=1-\frac{1}{f({v})+1}, {v}\in[0,1].
\label{def_eqn_g}
\end{equation}
According to Lemma \ref{lemma_f_property} that reveals the properties of $f({v})$, we have a similar lemma on $G({v})$:

\begin{lemma}
	Define $G({v})$ as Equation \ref{def_eqn_g}, and we have the following properties:
	\begin{enumerate}
		\item $G(0)=0, G(1)=0, 0<G({v})<1 for {v}\in(0,1)$.
		\item $G({v})$ is unimodal in $[0,1]$.
		\item $G'({v})=\frac{f'({v})}{(f({v})+1)^2}\Rightarrow|G'({v})|\leq 6C_f$.
	\end{enumerate}
	\label{lemma_g_property}
\end{lemma}
The proof of Lemma \ref{lemma_g_property} is trivial.

Notice that $G({v})$ is not necessarily a revenue curve, since $\frac{G({v})}{{v}}$ is not necessarily decreasing (and thus not a ``survival function''). However, we can construct a revenue curve $D({v}): [0,1]\rightarrow [0,1]$ via an affine transformation:
\begin{equation}
D({v})=\left\{
\begin{array}{lcl}
{v} &   & {v}\in[0,b]\\
b+(1-b)G(\frac{{v}-b}{1-b}) &   &{v}\in(b,1].
\end{array}
\right.
\label{def_eqn_D}
\end{equation}

Here $b=\frac{6C_f+1}{2}\in(0,1)$, and therefore $C_f<\frac16$. An illustration of the transformation from $f({v})$ (the upper figure) to $D({v})$ (the lower figure) is shown in Figure 1. The monotonicity in each interval is not changed, while the rate of $\frac{\E[r(v)]}{v}$ is non-increasing after these transformations.
As is mentioned above, the homothetic transformation with center $(1,1)$ ensures a non-increasing property of $\frac{D({v})}{{v}}$. Here we denote $d({v}):=\frac{D({v})}{{v}}$. To show that $D({v})$ is a revenue curve, we expand the definition of $d({v})$ to $\R$ as follows:
\begin{equation}
d({v})=\left\{
\begin{array}{lcl}
1 &   & {v}\in(-\infty,0]\\
\frac{D({v})}{{v}}&   & {v}\in(0,1)\\
0 &   &{v}\in[1,+\infty).
\end{array}
\right.
\label{def_eqn_d}
\end{equation}
Now we claim that there exists a random variable $Y>0$ such that $d({v})=\P[Y\geq {v}]$. To show this, it is sufficient to prove the following lemma:
\begin{lemma}
	For $d({v})$ defined in Equation \ref{def_eqn_d}, we have the following properties:
	\begin{enumerate}
		\item $d({v})$ is non-increasing on $\R$.
		\item $d({v})$ is continuous at ${v}=0$, i.e. $d(0)=\lim_{{v}\rightarrow0^{+}}d({v})=1$.
		\item $d({v})\geq 0, {v}\in[0,1]$.
	\end{enumerate}
\end{lemma}
\begin{proof}
	According to Equation \ref{def_eqn_D} and \ref{def_eqn_d}, we have:
	\begin{equation*}
	d({v})=\left\{
	\begin{array}{lcl}
	1 &   & {v}\in(-\infty,b]\\
	\frac b{v} + \frac{1-b}{{v}}\cdot{G}(\frac{{v}-b}{1-b})&   & {v}\in(b,1)\\
	0 &   &{v}\in[1,+\infty).
	\end{array}
	\right.
	\end{equation*}
	Therefore, we take the derivatives of $d({v})$ and get:
	\begin{equation*}
	\frac{\partial d({v})}{{v}}=\left\{
	\begin{array}{lcl}
	o &   & {v}\in(-\infty,0)\cup(0,b)\\
	-\frac{b-{v}\cdot G'(\frac{{v}-b}{1-b})+(1-b){G}(\frac{{v}-b}{1-b})}{{v}^2}&   & {v}\in(b,1)\\
	0 &   &{v}\in(1,+\infty).
	\end{array}
	\right.
	\end{equation*}
	Notice that
	\begin{equation*}
	\begin{aligned}
	b-{v}\cdot {G'}(\frac{{v}-b}{1-b})&\geq b - |{v}|\cdot|{G'}(\frac{{v}-b}{1-b})|\\
	&\geq b - 1\cdot 6C_f\\
	&=\frac{6C_f+1}{2}-6C_f\\
	&=\frac{1-6C_f}{2}\\
	&>0.
	\end{aligned}
	\end{equation*}
	The last inequality comes from the fact that $C_f<\frac16$. Therefore, $d'({v})$ is non-positive in $(-\infty,0)\cup(0,b)\cup(b,1)\cup(1,+\infty)$. Also, $d({v})$ is continuous at ${v}=0, {v}=b$ and $\lim_{{v}\rightarrow 1^{-}}=b>0=\lim_{{v}\rightarrow1^{+}}$, we know that $d({v})$ is always non-increasing on $\R$.
\end{proof}

Notice that there is a bijection between each $\{[a_k, b_k]\}_{k=0}^{\infty}$ series and each $f({v})$, and correspondingly each $G({v})$, $D({v})$ and $d({v})$. Still, a bijection lies between each $d({v})$ and the distribution $\P[Y\geq {v}]$ of the customers' valuation. Therefore, we will take $d({v})$ to represent this distribution. 

With all preparations done above, we are now able to prove Theorem \ref{theorem_lower_bound_main}. Specifically, we will proof the theorem on an infinite series of $n_1, n_2, \ldots$, where $n_k=\lceil\frac1k(\frac{w_{k-1}}{w_k^3})\rceil$. Consider the possible $Q_k=\frac{w_{k-1}}{3w_{k}}$ choices of $[a_k, b_k]$, and denote these intervals as $I_j, j=1,2,\ldots, Q_k$. If $[a_k, b_k]=I_j$, then we denote the corresponding $f({v}), G({v}), D({v}), d({v})$ functions as $f_j({v}), G_j({v}), D_j({v})$ and $ d_j({v})$ sequentially. Meanwhile, if we \emph{do not} make any choice of $[a_k, b_k]$, and then we just have a finite series of intervals  $[0,1]=[a_0,b_0]\supset[a_1,b_1]\supset[a_2,b_2]\supset\ldots\supset[a_{k-1}, b_{k-1}]$, and then we can define a $f_0({v})=C_f\cdot\sum_{j=0}^{k-1}w_j\cdot B_{[a_j, b_j]}({v})$, and can also define corresponding $G_0({v}), D_0({v}), d_0({v})$ based on $f_0({v})$.

Now, consider the pricing feedbacks in total $n$ rounds (where we denote $n_k$ as $n$ for simplicity). Define a \emph{feedback vector} $\vr_n\in\{0,1\}^n$, denoting the outcome of a deterministic policy interacting with the revenue curve. We claim that for $t=1,2,\ldots, n$, a vector $\vr_t$ is sufficient for any deterministic policy to generate a price ${v}_{t+1}$, because $\vr_i$ is a prefix of $\vr_j$ when $i\leq j$. For any policy $\pi$, denote the probability of $\vr_n$'s occurence as $\P_j(\vr_n)$ under the distribution $d_j$, or $\P_0(\vr_n)$ under the distribution $d_0$. Denote the series of prices that $\pi$ has generated as $\{{v}_t\}, t=1,2,\ldots, n$, and we may assume ${v}_t\geq b$ without losing generality (as $0\leq {v}<b$ is always suboptimal). Then, for any function $h: \{0,1\}^n\rightarrow[0,M]$, we have:
\begin{equation*}
\begin{aligned}
&\E_{\P_j}[h(\vr_n)]-\E_{\P_0}[h(\vr_n)]\\
=&\sum_{\vr_n}h(\vr_n)\cdot(\P_j[\vr_n]-\P_0[\vr_n])\\
\leq&\sum_{\vr_n:\P_{j}[\vr_n]\geq\P_0[\vr_n]}h(\vr_n)(\P_{j}[\vr_n]-\P_{0}[\vr_n])\\
\leq&M\cdot\sum_{\vr_n:\P_{j}[\vr_n]\geq\P_0[\vr_n]}h(\vr_n)(\P_{j}[\vr_n]-\P_0[\vr_n])\\
=&\frac{M}{2}\|\P_{j}-\P_0\|_1\\
\leq&\frac M 2\sqrt{2\ln{2}\cdot KL(\P_0||\P_j)}.\\
\end{aligned}
\end{equation*}
The last line comes from Lemma 11.6.1 in Cover \& Thomas, Elements of Information Theory, where $KL$ stands for the KL-divergence. Since
\begin{equation*}
\begin{aligned}
&\\KL(\P_0(\vr_n)||\P_{j}(\vr_n))
=&\sum_{t=1}^{n}KL(\P_0[r_t|\vr_{t-1}]||\P_j[r_t|\vr_{t-1}])\\
=&\sum_{t=1}^{t}\P_0(\frac{{v}_t-b}{1-b}\notin I_j)\cdot 0 + \P_0(\frac{{v}_t-b}{1-b}\in I_j)\cdot KL(\frac{D_{0}({v}_t)}{{v}_t}||\frac{D_j({v}_t)}{{v}_t}).
\end{aligned}
\end{equation*}
The first equality comes from the chain rule of decomposing a KL-divergence. The second equality is because$r_t$ is a Bernoulli random variable that satisfies $Ber(\frac{D_0({v}_t)}{{v}_t})$ under $\P_0$, or $Ber(\frac{D_j({v}_t)}{{v}_t})$ under $\P_j$. Denote ${\mu}_t:=\frac{{v}_t-b}{1-b}$ for simplicity. Notice that if ${v}_t\in I_j$, then we have:
\begin{equation*}
\begin{aligned}
\frac{D_j({v}_t)}{{v}_t}-\frac{D_0({v}_t)}{{v}_t}&=\frac{(1-b)\cdot G_j({\mu}_t)}{{v}_t}-\frac{(1-b)\cdot{G_0({\mu}_t)}}{{v}_t}\\
&=\frac{1-b}{{v}_t}(G_j({\mu}_t)-G_0({\mu}_t))\\
&=\frac{1-b}{{v}_t}\big(\frac{1}{f_0({\mu}_t)+1}-\frac{1}{f_j({\mu}_t)+1}\big)\\
&=\frac{1-b}{{v}_t}\frac{f_j({\mu}_t)-f_0({\mu}_t)}{(f_0({\mu}_t)+1)(f_j({\mu}_t)+1)}\\
&=\frac{1-b}{{v}_t}\frac{C_f\sum_{i=k}^{\infty}w_i\cdot B_{[a_i, b_i]}({\mu}_t)}{(f_0({\mu}_t)+1)(f_j({\mu}_t)+1)}\\
&\leq\frac{1-b}b\cdot\frac{C_f\cdot2w_k}{1\times 1}\\
&\leq{1}\cdot2C_fw_k\\
&\leq\frac{w_k}{3}
\end{aligned}
\end{equation*}
Here the third last inequality comes from ${v}_t\geq b, f_0({v})\geq 0, f_j({v})\geq 0$, and the fact that
\begin{equation*}
\sum_{i=k}^{\infty}w_iB_{[a_i, b_i]}({\mu}_t)\leq\sum_{i=k}^{\infty}3^{-i!}\cdot{1}\leq3^{-k!}\sum_{i=0}^{\infty}3^{-i}\leq\frac 23\cdot3^{-k!}<2w_k.
\end{equation*}
The second last inequality comes from $b=\frac{6C_f+1}2\geq\frac12$. The lastest inequality comes from the fact that $6C_f<1$.

Now we propose a lemma:
\begin{lemma}
	For Bernoulli distributions $Ber(p)$ and $Ber(p+\epsilon)$ with $\frac{1}{2}\leq p\leq p+\epsilon\leq \frac12 +C$, we have
	\begin{equation*}
	KL(p||p+\epsilon)\leq\frac{1}{\ln2}\frac{4}{1-4C^2} \epsilon^2.
	\end{equation*}
	
	\label{lemma_kl_bernoulli}
\end{lemma}
\begin{proof}
	\begin{equation*}
	\begin{aligned}
	KL(p||p+\epsilon)&=p\log(\frac{p}{p+\epsilon})+(1-p)log(\frac{1-p}{1-p-\epsilon})\\
	&=\frac{1}{\ln2}\cdot(p(-\ln(1+\frac{\epsilon}{p}))+(1-p)\ln(1+\frac{\epsilon}{1-p-\epsilon}))\\
	&\leq\frac{1}{\ln2}\cdot( p(-\frac{\epsilon}{p+\epsilon})+(1-p)\frac{\epsilon}{1-p-\epsilon})\\
	&=\frac{1}{\ln2}\cdot\frac{\epsilon^2}{(p+\epsilon)(1-p-\epsilon)}\\
	&\leq\frac{1}{\ln2}\cdot\frac{\epsilon^2}{(\frac{1}{2}+C)(\frac{1}{2}-C)}\\
	&\leq\frac{1}{\ln2}\cdot\frac{1}{\frac 14 - C^2}\epsilon^2.
	\end{aligned}
	\end{equation*}
	Here the third line comes from the fact that $\frac{{v}}{1+{v}}\leq \ln{{v}}\leq {v}$.
\end{proof}
Let us come back to the proof of the theorem. Since $\frac{D_0({v}_t)}{{v}_t}\geq b\geq \frac{1}{2}$, and the fact that 
\begin{equation*}
\begin{aligned}
\frac{D_j({v}_t)}{{v}_t}&\leq\frac{D_j(b+\frac{1-b}3)}{b+\frac{1-b}3}\\
&=3\cdot\frac{b+(1-b)(1-G_j(\frac 13))}{1+2b}\\
&=3\frac{b+(1-b)\frac{f_j(\frac13)}{f_j(\frac13)+1}}{1+2b}\\
&=3\frac{b+(1-b)\frac{C_f}{C_f+1}}{1+2b}.
\end{aligned}
\end{equation*}
The first inequality is because $\frac{D({v})}{{v}}$ is non-increasing and the fact that ${v}_t\geq b + \frac{1-b}3$ if ${v}_t\in[a_k, b_k], k\geq 1$. The last equality comes from the fact that $f_j(\frac13)=C_f$. Now we specify the constants: let $C_f=\frac{1}{60}, b=\frac{6C_f+1}2=\frac{11}{20}$. Plug in these constant values and we get:
\begin{equation*}
\frac{D_j({v}_t)}{{v}_t}\leq\frac{340}{427}<\frac56.
\end{equation*}

According to Lemma \ref{lemma_kl_bernoulli}, we have:
\begin{equation*}
KL(\frac{D_{0}({v}_t)}{{v}_t}||\frac{D_j({v}_t)}{{v}_t})\leq\frac{1}{\ln2}\cdot\frac{4}{1-4\cdot(\frac56-\frac12)^2}\cdot(\frac{w_k}3)^2=\frac{1}{\ln2}\cdot\frac{36}5\cdot\frac{w_k^2}9=\frac{1}{\ln2}\cdot\frac{4w_k^2}5.
\end{equation*}.
Recall that
\begin{equation*}
\begin{aligned}
&KL(\P_0(\vr_n)||\P_{j}(\vr_n))\\
=&\sum_{t=1}^{t}\P_0(\frac{{v}_t-b}{1-b}\in I_j)\cdot KL(\frac{D_{0}({v}_t)}{{v}_t}||\frac{D_j({v}_t)}{{v}_t})\\
\leq&\frac{1}{\ln2}\cdot\frac45w_k^2\cdot\sum_{t=1}^{n}\P_0[{\mu}_t\in I_j].
\end{aligned}
\end{equation*}
Therefore, we have:
\begin{equation*}
\begin{aligned}
&\E_{\P_j}[h(\vr_n)]-\E_{\P_0}[h(\vr_n)]\\
\leq&\frac{M}{2}\sqrt{2\ln2\cdot\frac{1}{\ln2}\cdot\frac45w_k^2\cdot\sum_{t=1}^{n}\P_0[{\mu}_t\in I_j]}\\
\leq&\frac{4M\cdot w_k}{5}\cdot \sqrt{\sum_{t=1}^{n}\P_0[{\mu}_t\in I_j]}.
\end{aligned}
\end{equation*}
Now , let $h(\vr_n)$ be $N_j=|\{t|{\mu}_t\in I_j, t=1,2,\ldots, n_k\}|$, and we know that $M=n_k$. Since $n_k=\lceil\frac1k\frac{w_{k-1}}{w_k^3}\rceil$, we conduct the pricing for $n_k$ times and have:
\begin{equation*}
\begin{aligned}
\E_{\P_j}[N_j]-\E_{\P_0}[N_j]\leq\frac{4M\cdot w_k}{5}\cdot \sqrt{\sum_{t=1}^{n}\P_0[{\mu}_t\in I_j]}=\frac{4M\cdot w_k}{5}\cdot \sqrt{\E_{\P_0}[N_j]}.
\end{aligned}
\end{equation*}
Sum over $j=1, 2,\ldots, Q_k$ of the inequality above, and we take an average to get:
\begin{equation}
\begin{aligned}
\frac{1}{Q_k}\cdot\sum_{j=1}^{Q_k}\E_{\P_j}[N_j]&\leq\frac{1}{Q_k}\sum_{j=1}^{Q_k}\E_{\P_0}[N_j]+\frac{1}{Q_k}\frac45 n_k\cdot w_k\sum_{j=1}^{Q_k}\sqrt{\E_{\P_0}[N]}\\
&=\frac1{Q_k}\cdot{n_k} +\frac1{Q_k}\frac45n_k\cdot w_k\sum_{j=1}^{Q_k}\sqrt{\E_{\P_0}[N_j]}\\
&\leq\frac{n_k}{Q_k}+\frac45\frac{n_k}{Q_k}\cdot w_k\cdot\sqrt{Q_k\cdot\sum_{j=1}^{Q_k}\E_{\P_0}[N_j]}\\
&=\frac{n_k}{Q_k}+\frac45\frac{n_k}{Q_k}\cdot w_k\cdot\sqrt{Q_kn_k}\\
&\leq\frac3k\cdot\frac45\frac3k\frac1{w_k^2}\sqrt{\frac3{k^2}\cdot\frac{w_{k-1}^2}{w_k^4}}\\
&=\frac3k\frac1{w_k^2}+\frac{4\sqrt3}{5\sqrt{k}}\frac1k\frac{w_{k-1}}{w_k^3}\\
&\leq0.9\cdot n_k, \text{ for $k\geq 3$.}
\end{aligned}
\label{eqn_const_frac}
\end{equation} 
In Equation \ref{eqn_const_frac}, the first line comes from the summation; the second (and the fourth) line is because $\sum_{j=1}^{Q_k}\E_{\P_0}[N_j]=\E_{\P_0}[\sum_{j=1}^{Q_k}N_j]=n_k$; the third line is an application of Cauchy-Schwartz's Inequality; the fifth line is derived by plugging in $Q_k=\frac{w_{k-1}}{3w_k}, n_k\leq\frac1k\cdot\frac{w_{k-1}}{w_k^3}, w_k=3^{-k!}$; the last line is just calculations. Therefore, under distribution $d_j$, the policy $\pi$ is expected to choose an ${v}_t\notin I_j$ for at least $0.1n_k$ times, which will bring a regret $0.1n_k\cdot C_j\cdot w_k=\frac1{600}n_k\cdot w_k=\frac1k\cdot w_k^{\frac1k-2}$. Since $n_k=\frac1k\cdot w_k^{\frac1k-3}$, we know that $Regret=\Omega((n_k)^{\frac23-\frac{1}{3k}})$ up to logarithmic factors. Therefore, we claim that for any $\delta>0$, no policy can achieve $o(n_k^{\frac23-\delta})$ for sufficiently large $k$.


This ends the proof of Theorem \ref{theorem_lower_bound_main}.


\section{More Discussions}
In this work, we have developed two ``linear'' approaches toward the agnostic dynamic pricing problem. There are, however, still some issues that we have a handful of insights instead of rigorous proof or empirical evidence. Here we would like to present these ideas that might serve as heuristics for further research.

\subsection{Differences between LP and LV}
\label{appendix_lp_vs_lv}
As we stated in Section \ref{sec_intro}, LP models our strategy while LV modes the nature. Also, a good (no-regret) LP algorithm approaches the best linear policy in total while a good LV algorithm approaches the global optimal price at each round. When we adopt a LP problem model, we indeed have very little information about the market valuation other than obvious features of the product to sell. In this situation, a linear pricing policy is tractable and transparent to the customers, but it is not guaranteed to present or approach the best price. When we adopt a LV problem model, it is assumed that we have already known all features of the selling session (not limited to the product itself), and the fluctuation caused by the market is independent to the product. Therefore, we may learn from the feature-pricing-feedback data over the time and estimate the noise distribution, which would help approaching the best price combining with a greedy policy. Here is a concrete example regarding vehicle owners, dealers and buyers that illustrates the difference between LP and LV:

In Session 1, suppose we are the owner and would like to sell our used car to a buyer/dealer. A 3rd-party evaluator will evaluate your car based on a few (but not all) factors, e.g., mileage, duration, condition and accident records, and then subtract a certain amount from the selling price of an identical new car. This amount is usually linearly or near-linearly dependent on these factors listed above. Remember that this selling price is proposed by we owners. In other words, we are the seller in this session, and the buyer/dealer would respond by accepting or declining the price we propose. Here we adopt a linear pricing policy because we do not have full information of the selling session, and therefore customers' valuation model is indeed unclear to us.

In Session 2, suppose we are the dealer and would like to sell a used car to a buyer. Car dealers usually have sufficient information on the vehicle and the market supply-demand relationship. At least, we know clearly about which features are related to customers' valuations. Therefore, it is reasonable for usr to assume a parametric noisy valuation model (possibly a LV model) on their customers, and we would optimize these parameters based on historical selling records. With the model being well-learned, we may approach the global optimal price every time. That we directly make assumptions on customers' valuation model is reasonable since we dealers have sufficient information, but this could still be risky if the features we can observe are limited.

\subsection{Applying LP Algorithm to LV Problem Model}
LP and LV's are two distinctly different problems: the optimal prices in an LV problem is not necessarily linear w.r.t. $x_t$: when $x_t=0$ with a zero-mean noise, the expected reward of a price that is slightly larger than 0 would be positive while the expected reward of 0 price is exactly 0. Therefore, we believe that an optimal linear policy would suffer an $\Omega(T)$ regret in some LV settings even with known noise distributions. However, if the noise distribution is parametric by some parameter $\eta\in\R^{k}$, then we might have a ``pseudo-linear'' policy $\tilde{\eta}:[\theta_1, \theta_2, \ldots, \theta_d, \eta_1, \eta_2, \ldots, \eta_k]^{\top}$ that takes $\tilde{x}:=[x_1, x_2, \ldots, x_d, \phi_1(u), \phi_2(u), \ldots, \phi_k(u)]^{\top} $ as input, and outputs $v = \tilde{\eta}^{\top}\tilde{x}$. A similar linearization idea in non-feature pricing has been adopted in \citet{wang2021multimodal} and achieves optimal regrets. However, it is still unknown whether their methods can be applied to this feature-based LV problem. The key to this approach is to figure out a (nearly-)linear action-to-reward mapping, but this seems really hard in this setting. Again, in an LV problem it is the valuation instead of the optimal price that is linear. 
\subsection{The Hardness of Pricing versus Bandits}
The generic feature-based dynamic pricing problem can be reduced to a contextual bandit problem with continuum action and infinite policy spaces, despite some literature that assumes a different acceptance/declination reward scope (see \citet{bartok2014partial}). Therefore, the gap between a dynamic pricing problem and an ordinary (discrete-action and finite-policy) contextual bandit problem can be observed from three perspectives. Firstly, the pricing feedback contains more information than a bandit feedback: if $v_t$ is accepted, then any $v\leq v_t$ would have been accepted if it were proposed. We call this a ``half-space information''. Secondly, a discrete action space might not contain the optimal or any near-optimal price that matches the minimax regret: the revenue curve can vary drastically with respect to the price (e.g., consider a noise whose pdf is a rescaled Weierstrass function). Thirdly, a finite policy space might not contain the global optimal or any near-optimal policy, either. This is possible even for a parametric policy space where the parameter space is infinite. Therefore, we cannot directly adopt the regret bounds of contextual bandits onto feature-based dynamic pricing problems unless there exists a rigorous reduction.

However, we notice that the three perspectives above are pointing at different directions: the ``half-space information'' makes pricing easier than bandits, while the other two discretization issues makes it harder. In fact, we might partially offset the ``continuum action'' issue with the ``half-space information'' just like what we did in this paper: the revenue curve is actually ``half Lipschitz'' that $g(v_1,u,F)-g(v_2, u, F)\leq v_1-v_2$ if $v_1\geq v_2$. This helps our algorithms get rid of the Lipschitz assumption. However, this is not rich enough to substantially reduce the regret as we still use bandit algorithms to achieve a minimax rate in an LP problem, where the lower bound holds even for Lipschitz revenue curve. Therefore, a very important question occurs to us: what else could a pricing feedback provide other than the ``half Lipschitz''? Technically speaking, does a pricing feedback contain high-order information of the revenue curve? Besides, remember that we still do not have a unified approach toward a finite near-optimal policy set. In this work, we discretize the noise distribution by $\frac1{\gamma}$ grids, which indeed increases the regret bound. For more sophisticated feature-valuation mapping (e.g., a non-linear valuation model) that is hard to parameterize, maybe it is not suitable to just apply naive discretization methods.

As a result, pricing problem seems at least as hard as bandits, and it is still unclear whether or not we could completely solve the feature-based dynamic pricing via contextual bandit methods (even though the major contributions toward single-product dynamic pricing are from multi-armed-bandits-related approaches). 

\subsection{Social Impacts}
In this work, we mainly focus on an online-fashion pricing problem where only one product is sold to one customer at each round (time spot). Therefore, it is not likely to commit a pricing discrimination according to its rigorous definition (since the price fluctuation over time should not be treated as discrimination). However, there exist chances that our algorithm could be misused. Notice that each item is characterized by a feature vector $x_t$, which might be used to capture more information, e.g., customers' behaviors. On the one hand, it is indeed a price discrimination if we propose differently-generated prices to customers with different personal features even at different time point as long as the market has not changed substantially. On the other hand, this would lead to a potential leakage of personal privacy. It is usually forbidden to collect and use personal information for commercial use, but the sellers would at least know what the customers have bought and how much they have paid. Even though the feature $x_t$ can be encoded with cryptographic techniques such that it is still suitable for learning (e.g., a ``fully-homomorphic encryption'', or FHE), at least the proposed prices are informative and might reveal the customer's behaviors. Indeed, auctions are a method to avoid any pricing discrimination, but it is not practical in most of the situations happening in our daily life.

\end{document}